\documentclass[journal]{IEEEtran}

\ifCLASSINFOpdf

\else

\fi

\usepackage{times}
\usepackage{epsfig}
\usepackage{graphicx}
\usepackage{amsmath}
\usepackage{amssymb}
\usepackage{makecell}
\usepackage{enumitem}
\usepackage{multirow}
\usepackage{xcolor}
\usepackage{graphicx}
\usepackage{pifont}
\usepackage{comment}
\newcommand{\xmark}{\ding{55}}%
\newcommand{\cmark}{\ding{51}}%
\usepackage{footnote}
\usepackage[symbol]{footmisc}

\usepackage{graphicx}
\usepackage{scalerel, amsmath}
\usepackage{amssymb}
\usepackage{booktabs}
\usepackage{verbatim}
\usepackage{pifont}
\usepackage{dsfont}
\usepackage{float}
\usepackage[english]{babel}
\usepackage{amsthm}
\usepackage{xcolor}
\usepackage{multicol}
\usepackage{multirow}
\usepackage{listings}
\usepackage{algorithm}
\usepackage{float}

\usepackage{subcaption}

\usepackage{algpseudocode}

\newtheorem{lemma}{Lemma}

\begin{document}
\title{Fairness in Visual Clustering: A Novel Transformer Clustering Approach}

\author{Xuan-Bac Nguyen,~\IEEEmembership{Student Member,~IEEE,}
        Chi Nhan Duong,~\IEEEmembership{Member,~IEEE,}
        Marios Savvides,~\IEEEmembership{Member,~IEEE,}
        Kaushik Roy,~\IEEEmembership{Member,~IEEE,}
        Hugh Churchill,
        and Khoa Luu$^{**}$,~\IEEEmembership{Member,~IEEE,}%
\thanks{Xuan-Bac Nguyen, and Khoa Luu are with the Department of Computer Science and Computer Engineering, University of Arkansas, Fayetteville, USA}%
\thanks{Hugh Churchill is with the Department of Physics, University of Arkansas, Fayetteville, USA}
\thanks{Chi Nhan Duong is with the Department of Computer Science and Software Engineering, Concordia University, Montreal, Canada}%
\thanks{Marios Savvides is Carnegie Mellon University, USA}
\thanks{Kaushik Roy is with North Carolina A\&T State University, USA}
\thanks{$^{**}$ Corresponding Author: Khoa Luu, email: khoaluu@uark.edu}
}

\markboth{IEEE TRANSACTIONS ON IMAGE PROCESSING}%
{Nguyen \MakeLowercase{\textit{et al.}}: Fairness in Visual Clustering: A Novel Transformer Clustering Approach}

\maketitle

\begin{abstract}
Promoting fairness for deep clustering models in unsupervised clustering settings to reduce demographic bias is a challenging goal. It is because of the limitation of large-scale balanced data with well-annotated labels for sensitive or protected attributes. In this paper, we first evaluate demographic bias in deep clustering models from the perspective of cluster purity, measured by the ratio of positive samples within a cluster to their correlation degree. This measurement is adopted as an indication of demographic bias. Then, a novel loss function is introduced to encourage a purity consistency for all clusters to maintain the fairness aspect of the learned clustering model.
Moreover, we present a novel attention mechanism, Cross-attention, to measure correlations between multiple clusters, strengthening faraway positive samples and improving the purity of clusters during the learning process. 
Experimental results on a large-scale dataset with numerous attribute settings
have demonstrated the effectiveness of the proposed approach on both clustering accuracy and fairness enhancement on several sensitive attributes. 
\end{abstract}


\IEEEpeerreviewmaketitle

\section{Introduction}

Unsupervised learning in automated object and human understanding has recently become one of the most popular research topics. It is because of the nature of the extensive collection of available raw data without labels and the demand for consistent visual recognition algorithms across various challenging conditions. 
Standard visual recognition, e.g., Face Recognition \cite{wang2018cosface} or Visual Landmark Recognition \cite{weyand2020google}, has recently achieved high-performance accuracy in numerous practical applications where probe and gallery visual photos are collected in different environments. 

Together with accuracy, algorithmic fairness has recently received broad attention from research communities as it may enormously impact applications deployed in practice. For example, a face recognition system giving impressive accuracy on white faces while having a very high false positive rate on non-white faces could result in unfair treatment of individuals across different demographic groups. Therefore, by defining a sensitive (protected) attribute (such as gender, ethnicity, or age), a fair and accurate recognition system is particularly desirable.

\begin{figure}[!t]
    \centering  
    \includegraphics[width=0.95\columnwidth]{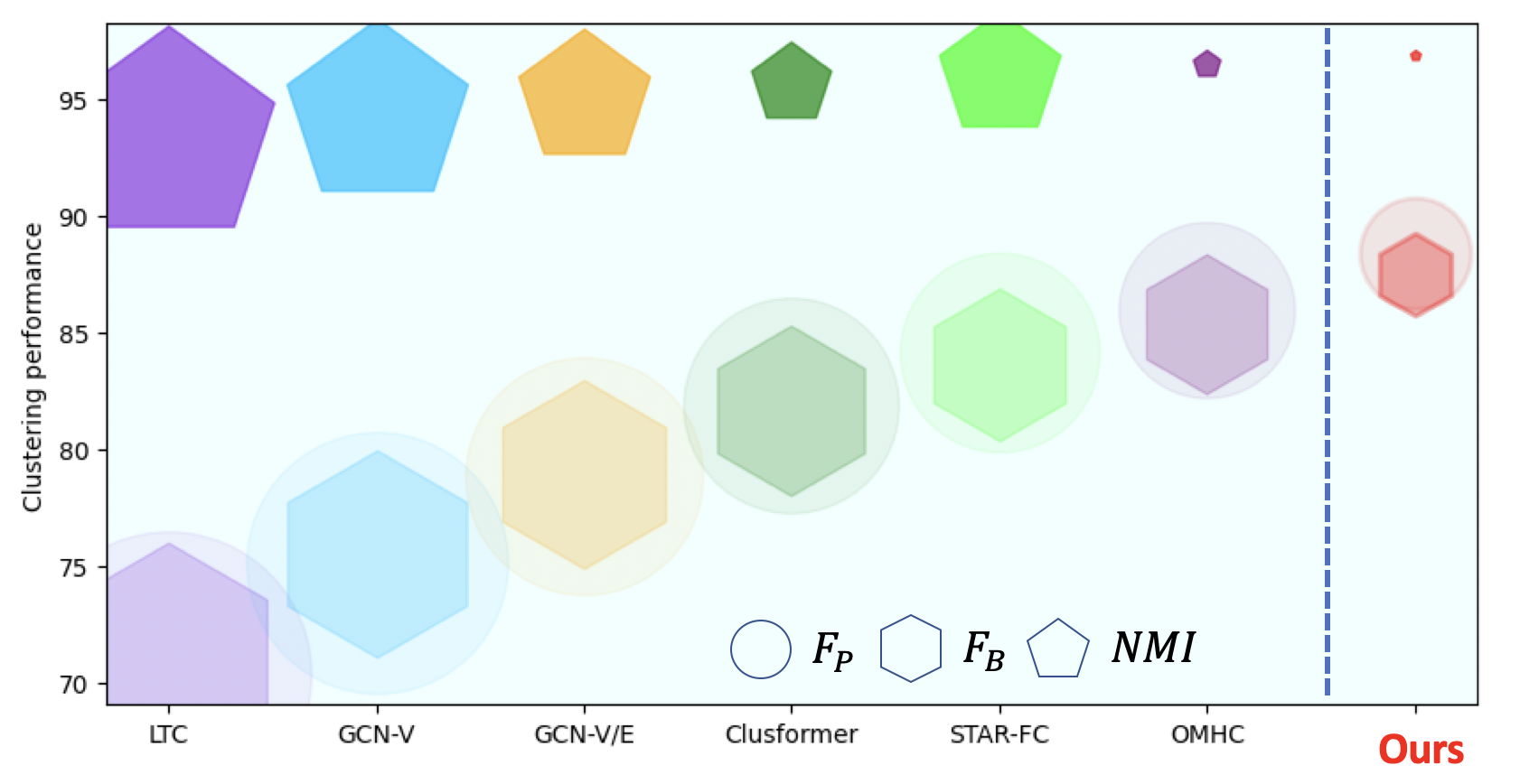}
    \caption{Clustering performance of BUPT-BalancedFace on ethnicity. Each shape area is proportional to \textit{standard deviation} of performance on different groups of the race, i.e., Asian, African, Caucasian, and Indian. The smaller shape, the better the fairness method is. Our method achieves both clustering accuracy and fairness on all metrics ($F_P$, $F_B$, and $NMI$). \textbf{Best viewed in color.}}
    \label{fig:Fig1}
    \centering
\end{figure}

To alleviate the demographic bias for better fairness, many efforts have been directed toward deep learning approaches to extract less biased representations. A straightforward approach is to collect balanced datasets with respect to the sensitive attributes \cite{robinson2020face, wang2019racial, wang2020mitigating, hupont2019demogpairs} %
and use them for the training process. 
However, this approach requires extreme efforts to collect relatively balanced samples and annotations for sensitive attributes. Otherwise, an insufficient number of training samples may also lead to reducing the accuracy of learned models.
Interestingly, approaches relying on balanced training data still suffer from demographic bias to some degree \cite{wang2020mitigating}. %
Another approach is to focus on algorithmic designs such as Fairness-Adversarial Loss \cite{li2020deep}, maximizing the conditional mutual information (CMI) between inputs and cluster assignment \cite{zeng2022deep} given sensitive attributes, Domain Adaptation \cite{mukherjee2022domain, wang2021towards}, 
or deep information maximization adaptation network by transferring recognition knowledge from one demographic group to other groups \cite{wang2020mitigating, wang2019racial}.
Although these introduced methods have shown their advantages in reducing demographic bias without requirements of balanced data, they still lack generalization capability (i.e., learning for a particular sensitive attribute). Moreover, they also rely on accurate annotations for that sensitive attribute during learning.

In this paper, our goal focuses on promoting fairness for deep clustering models in unsupervised clustering settings. Motivated by the fact that clusters of major demographic groups with many samples consist of only a few negative samples and that their correlations are very strong while those of minor demographic groups usually contain a large number of noisy or negative samples with weak correlations among positive ones, the \textbf{\textit{cluster purity}}, i.e., \textit{ratio of the number of positive samples within a cluster to their correlations}, plays a vital role in measuring clustering bias across demographic groups.
Therefore, we first evaluate the bias in deep clustering models from the perspective of cluster purity. Then, by encouraging consistency of the purity aspect for all clusters, demographic bias can be effectively mitigated in the learned clustering model.
In the scope of this work, we further assume that the annotations for sensitive attributes are inaccessible due to their expensive collecting efforts and/ or privacy concerns in annotating them.

\noindent
\textbf{\textit{Contributions.}} In summary, our contributions are four-fold.
(1) The cluster purity and correlation between positive samples are first analyzed and adopted to indicate a demographic bias for a visual clustering approach. 
(2) By promoting the purity similarity across clusters, the visual clustering fairness is effectively achieved without requirements of sensitive attributes' annotations. This property is formed into a novel loss function to improve fairness with respect to various kinds of sensitive attributes.
(3) Regarding deep network design, we present a novel attention mechanism, termed Cross-attention, to measure correlations between multiple clusters and help faraway samples have a stronger relationship with the centroid.
(4) Finally, by comprehensive experiments, the proposed approach consistently achieves State-of-the-Art (SOTA) results compared to recent clustering methods \cite{Liu_2021_ICCV,Shen_2021_CVPR,Nguyen_2021_CVPR,yang2019learning,yang2020learning} on a standard large-scale visual benchmark \cite{Wang_2019_ICCV} across various demographic attributes, i.e., ethnicity, age, gender, race as shown in the Fig \ref{fig:Fig1}.

\section{Related Work}
\noindent
\textbf{Deep Clustering.} 
In most recent studies of deep clustering, clustering functions are often Graph Convolution Neural Networks (GCN) \cite{yang2019learning,yang2020learning} or Transformer-based Networks \cite{Nguyen_2021_CVPR}. Both approaches share the same goal: determine which samples in a cluster are different from the centroid and remove them from the cluster. 

In GCN-based methods, each cluster is treated as a graph where each vertex represents a sample, and an edge connecting two vertices illustrates how strong the connection between them is. The GCN method aims to explore the correlation of a vertex with its neighbors. %
Yang et al. \cite{yang2020learning} presented a method to estimate confidence (GCN-V) of a vertex being a true positive sample or outlier of a given cluster. In addition, GCN-E is also introduced to predict if two vertexes belong to the same class (or cluster). 
However, in minor clusters where purity is low and connections between vertices are weak, the network becomes ineffective \cite{Nguyen_2021_CVPR}. Other GCN-based approaches \cite{Shen_2021_CVPR,Liu_2021_ICCV} also face the same issue.
Transformer has been applied to many tasks in Natural Language Processing (NLP) \cite{radford2018improving,devlin2018bert,dai2019transformer,yang2019xlnet} and computer vision \cite{Yue_2021_ICCV,Arnab_2021_ICCV,Liu_2021_ICCV,Chen_2021_ICCV,Graham_2021_ICCV,Chen_2021_ICCV,Zou_2021_CVPR,Wang_2021_CVPR,Chen_2021_CVPR,nguyen2023micron,truong2022conda,truong2023fredom,phacounting,nguyeniccv,nguyenssl}. Nguyen et al. \cite{Nguyen_2021_CVPR} introduced a new Transformer based architecture for visual clustering. 
This method treats a cluster as a sequence 
starting with a cluster's centroid and followed by its neighbors with decreasing order of similarities.
Then, Clusformer predicts 
which vertex is the same class as the centroid. Chen et al. \cite{omhc} 
leverage transformer-based approach as the feature encoder for face clustering.  

\noindent
\textbf{Fairness in Computer Vision.}
Fairness has garnered much attention in recent computer vision and deep learning research \cite{Xu_2021_CVPR, buolamwini2018gender, chen2021understanding, garcia2019harms, hazirbas2021towards, jung2022learning, li2019repair, misra2016seeing, terhorst2020face, wang2022measuring, wang2019racial, wang2019balanced, yang2022explaining}. The most common objective 
is to improve fairness by lowering the model's accuracy disparity between images from various demographic sub-groups. Facial recognition is one of the most common topics, attracting numerous researchers. There are notable datasets for fair face recognition such as the Diversity in Faces (DiF) dataset \cite{merler2019diversity}, Racial Faces in-the-Wild (RFW), BUPT-GlobalFace, BUPT-BalancedFace \cite{Wang_2019_ICCV, wang2021meta}.
Fairness in the visual clustering problem is also addressed in \cite{deep_fair_clustering, deep_clustering_based_fair_outlier_detection}. Li et al. \cite{deep_fair_clustering} 
proposed a deep fair clustering method to hide sensitive attributes. %
Song et. al \cite{deep_clustering_based_fair_outlier_detection} focused on 
fairness regarding unsupervised outlier detection and proposed a Deep Clustering based Fair Outlier Detection framework (DCFOD) to simultaneously detect the validity and group fairness. Experiments were conducted on the relatively small MNIST database. These prior works \cite{deep_fair_clustering, deep_clustering_based_fair_outlier_detection} were not designed to be robust and scalable. Training these models is expensive as they require demographic attributes costly to annotate. In addition, these clustering approaches are not practical when dealing with unseen/unknown subjects as these methods require defining the number of clusters prior.

To the best of our knowledge, this is the first work dealing with \textit{large-scale clustering with the fairness criteria without involving demographic attributes during training}.

\section{Problem Formulation}

Let $\mathcal{D}$ be a set of $N$ data points to be clustered. We define $\mathcal{H}: \mathcal{D} \rightarrow \mathcal{X} \subset \mathbb{R}^{N \times d}$ as a function that embeds these data points to a latent space $\mathcal{X}$ of $d$ dimensions. 
Let $\hat{Y}$ be a set of ground-truth cluster IDs assigned to the corresponding data points. Since $\hat{Y}$ has a maximum of $N$ different cluster IDs, we can represent $\hat{Y}$ as a vector of $\left[0, N-1\right]^{N}$. 
A deep clustering algorithm $\Phi: \mathcal{X} \rightarrow Y \in \left[0, N-1\right]^{N}$ is defined as a mapping that maximizes a clustering metric $\sigma(Y, \hat{Y})$.
As $\hat{Y}$ is not accessible during the training stage, the number of clusters and samples per cluster are not given beforehand.
Therefore, \textit{divide and conquer} philosophy can be adopted to relax the problem. In particular, a deep clustering algorithm $\Phi$ can be divided into three sub-functions: \textit{pre-processing} $\mathcal{K}$, \textit{in-processing} $\mathcal{M}$, and \textit{post-processing}  $\mathcal{P}$ functions. Formally, $\Phi$ can be presented as in Eqn. \eqref{eqn:ClusteringAlgorithm}.
\begin{equation} \label{eqn:ClusteringAlgorithm}
    \Phi(\mathcal{X}) = \{\mathcal{M} \circ \mathcal{K}(\mathbf{x}_i, n)\}_{\mathbf{x}_i \in \mathcal{X}}
\end{equation} 
where $\mathcal{K}$ denotes an unsupervised clustering algorithm on deep features $\mathbf{x}_i$ of the $i$-th sample. $\text{k-NN}$ is usually a common choice for $\mathcal{K}$ \cite{yang2020learning,yang2019learning,Shen_2021_CVPR,omhc,Nguyen_2021_CVPR}.  
In addition, for the post-processing $\mathcal{P}$ function, it is usually a rule-based algorithm to merge two or more clusters into one when they share high concavities \cite{yang2019learning,yang2020learning,Shen_2021_CVPR}. Since $\mathcal{P}$ does not depend on the input data points nor is a learnable function,
the parameters $\theta_{\mathcal{M}}$ of $\Phi$ can be optimized via this objective function:
\begin{equation}
    \label{eq:clustering}
    \theta^*_{\mathcal{M}} = \arg \min_{\theta_{\mathcal{M}}} \mathbb{E}_{\mathbf{x}_i \sim p(\mathcal{X})} \left[ \mathcal{L} (\mathcal{M} \circ \mathcal{K}(\mathbf{x}_i, n), \hat{q}_i) \right]
\end{equation}
where $k$ is the number of the nearest neighbors of input $\mathbf{x}_i$, $\hat{q_i}$ is the objective of $\mathcal{M}$ to be optimized, and $\mathcal{L}$ denotes a suitable clustering or classification loss.

Let $G = \{g_1, g_2, \dots g_p\}$ be a set of sensitive attributes (i.e., race, gender, age). 
We further seek predictions for $\Phi$ to be accurate and fair to $G$. In other words, a \textit{fairness criteria} requires $Y$ not to be biased in favor of one attribute or another, i.e., 
$P(Y \vert G = g_1) = P(Y \vert G = g_2) = \dots = P(Y \vert G = g_p)$. 
To achieve this goal, we define $\mu_i =  \mathbb{E}_{\mathbf{x}_i \sim p(g_i)}\left[\mathcal{L^{\textit{fair}}} (\mathcal{M} \circ \mathcal{K}(\mathbf{x}_i, n), \hat{q}_i) \right]$ as the expected value of loss function $\mathcal{L^{\textit{fair}}}$ from data points belong to a group with the sensitive attribute $g_i$, where $\mathcal{L}^{\textit{fair}}$ is a suitable fairness loss. Then, the fairness discrepancy between groups can be defined as in Eqn. \eqref{eq:fairness}.
\begin{equation} \label{eq:fairness}
\begin{split}    
    \Delta_{DP}(\Phi) = \frac{1}{2}\left(\sum_{i=1}^{p} \sum_{j=1}^{p} \vert\mu_i - \mu_j\vert\right)
\end{split}
\end{equation}
In order to produce a fair prediction across sensitive attributes, $\Delta_{DP}(\Phi)$ should be minimized during the learning process. Therefore, from Eqn. \eqref{eq:clustering} and Eqn. \eqref{eq:fairness}, the objective function with fairness factor can be reformulated as:
\begin{equation}
\label{eq:fairness_object_tive}
\begin{split}
\theta^*_{\mathcal{M}} &= \arg \min_{\theta_{\mathcal{M}}} \mathbb{E}_{\mathbf{x}_i \sim p(\mathcal{X})} \left[ \mathcal{L} ( \mathcal{M}\circ \mathcal{K}(\mathbf{x}_i, n), \hat{q}_i) \right] + \lambda \Delta_{DP}(\Phi)  \\
&= \arg \min_{\theta_{\mathcal{M}}} \mathbb{E}_{\mathbf{x}_i \sim p(\mathcal{X})} \left[ \mathcal{L} (\mathcal{M} \circ \mathcal{K}(\mathbf{x}_i, n), \hat{q}_i) \right] \\ 
& \quad \quad \quad \quad \quad \quad \quad \quad \quad +   \frac{\lambda}{2}\left(\sum_{i=1}^{p} \sum_{j=1}^{p} \vert\mu_i - \mu_j\vert\right)
\raisetag{50pt}
\end{split}
\end{equation}

\noindent
\textbf{Theoretical Analysis.} From the definition of $\mu_i$, since $\mathcal{L}^{\textit{fair}} (\mathcal{M}\circ \mathcal{K}(\mathbf{x}, n), \hat{q}_i)$ is a loss function (i.e., Binary-cross entropy), the expectation of this function is greater than 0. 

\begin{lemma}
\label{lemma_1}
$\left\vert\mu_i - \mu_j\right\vert \leq \mu_i + \mu_j \quad \quad \forall \mu_i, \mu_j \geq 0$
\end{lemma}
\begin{proof} Let $\mu_i, \mu_j \geq 0$, we have
\begin{equation}
\begin{split}
&0 \leq \left\vert\mu_i - \mu_j\right\vert \leq \mu_i + \mu_j  \\
\iff &\left(\left\vert\mu_i - \mu_j\right\vert\right)^{2} \leq \left(\mu_i + \mu_j\right)^{2} \\
\iff &\mu_i^{2} - 2\mu_i \mu_j + \mu_j^2 \leq \mu_i^{2} + 2\mu_i \mu_j + \mu_j^2   \\
\iff &-\mu_i \mu_j \leq \mu_i \mu_j
\end{split}
\end{equation}
\raisetag{50pt}
\end{proof}
\noindent
Applying the Lemma \ref{lemma_1} to Eqn (\ref{eq:fairness_object_tive}), a new fairness objective function can be derived as,
 \begin{equation} \label{eq:upper_bound_objective_function}
 \begin{split}
     \theta^*_{\mathcal{M}} 
     &\leq \arg \min_{\theta_{\mathcal{M}}} \mathbb{E}_{\mathbf{x}_i \sim p(\mathcal{X})} \left[ \mathcal{L} (\mathcal{M} \circ \mathcal{K}(\mathbf{x}, n), \hat{q}_i) \right] 
     + \lambda \sum_{i=1}^{p} \mu_i \\
    &\leq \arg \min_{\theta_{\mathcal{M}}} \underbrace{\mathbb{E}_{\mathbf{x}_i \sim p(\mathcal{X})} \left[ \mathcal{L} (\mathcal{M} \circ \mathcal{K}(\mathbf{x}_i, n), \hat{q}_i) \right]}_{O1: \quad \text{Maximize clustering performance}} \\ 
     & \quad \quad + \lambda \sum_{i=1}^{p} \underbrace{\mathbb{E}_{x_i \sim p(g_i)}\left[\mathcal{L}^{\textit{fair}} (\mathcal{M} \circ \mathcal{K}(\mathbf{x}_i, n), \hat{q}_i) \right]}_{O2: \quad \text{Maximize \textbf{fairness}}}
 \end{split}
 \raisetag{50pt}
 \end{equation}

In practice, each sensitive distribution $p(g_i)$ is unknown, expensive to measure, or even shifts over time. 
Moreover, some sensitive attributes have large and dominant samples, while the minor groups have few samples. %
Thus, optimizing the second objective $\left( O2 \right)$ and maintaining fairness in clustering remains challenging.
Subsection \ref{subsec:bias} analyzes the biases in clustering and their relationships to cluster purity.

\subsection{Bias in The Clustering and its effects} \label{subsec:bias}
\noindent
Since training data points $\mathcal{D}$ for a clustering approach are collected in some limited environments and constraints, the distributions of sensitive attributes are usually imbalanced.
This data type often leads to unfairness in every learning stage and the overall system.
Let $D_{major}$ and $D_{minor}$ represent the samples in the major and minor attributes, respectively.
Because $\vert\mathcal{D}_{major}\vert \gg \vert\mathcal{D}_{minor}\vert$, a feature extractor $\mathcal{H}$ trained on $\mathcal{D}$ tends to generate more discriminative features for $D_{major}$ than $D_{minor}$.

\noindent
\textit{\textbf{Bias in Cluster Purity.}} A domino effect from $\mathcal{D}$ to $\mathcal{H}$ also leads to another bias on the predicted clusters, namely \textbf{\textit{cluster purity}}. Particularly, latent features $\mathbf{x}_i \in \mathcal{X}$ are employed to $\mathcal{K}$ to construct predicted clusters denoted as $\mathcal{N}(\mathbf{x}_i) = \mathcal{M} \circ \mathcal{K}(\mathbf{x}_i, n)$. As latent features of samples with sensitive attributes belonging to $D_{major}$ are discriminative, their correlations are quite strong. Therefore, their constructed cluster has many positive samples, making its purity very high. Meanwhile, weak connections between samples of $D_{minor}$ may reflect in many noisy samples within a cluster, causing very low-quality clusters.
Let the cluster purity be the ratio of the number of true positive samples within a predicted cluster:
\begin{equation}
    \label{eq:purity}
    \gamma_i(\mathbf{x}_i) = \frac{\vert\mathcal{N}^{+}(\mathbf{x}_i)\vert}{n - n^{-}}
\end{equation}
where $\mathcal{N}^{+}(\mathbf{x}_i) = \{ \{\mathbf{x}_{j}\}\vert \mathbf{x}_{j} \in \mathcal{N}(\mathbf{x}_i) \land y_j = y_i\}$ denotes the true positive set in the predicted cluster $\mathcal{N}(\mathbf{x}_i)$, and $n^{-}$ is the number of predicted negative samples. $y_i, y_j \in Y$ are the corresponding cluster IDs. Noted that the cluster purity $\gamma_i(\mathbf{x}_i)$ provides the rate of correctly predicted samples and indicates the rate of positive samples that cannot be identified by $\Phi$. Due to the bias,
 $\frac{1}{\vert\mathcal{D}_{major}\vert} \sum_{\mathbf{x}_i}^{\mathcal{X}_{major}} \gamma_i(\mathbf{x}_i) \gg \frac{1}{\vert\mathcal{D}_{minor}\vert} \sum_{\mathbf{x}_i}^{\mathcal{X}_{minor}} \gamma_i(\mathbf{x}_i)$.

In general, directly mitigating this kind of bias via a \textit{``perfect balanced training dataset''} is infeasible due to (1) a considerable effort to collect millions of images and their annotations for sensitive attributes; and (2) numerous enumerations of various sensitive attributes such as race, age, gender. 
Therefore, rather than focusing on constructing a balanced training dataset for different sensitive attributes, we proposed a \textit{penalty loss to promote the purity consistency of all clusters across demographic groups}. In this way, the rate of correctly predicted positive samples, as well as the missing positive samples, can be maintained to be similar among all clusters in both $\mathcal{D}_{minor}$ and $\mathcal{D}_{major}$, and, therefore, enhance the fairness in model's predictions. 
In addition, we further proposed an \textit{Intraformer architecture to encourage more robust connections between positive samples} that are far away from the cluster's centroid. It can effectively enhance the purity of hard clusters belonging to minority groups.

\section{Clustering Purity Penalty Loss} \label{sec:puritylossa}
\subsection{Clustering Accuracy Penalty}

In previous methods, \cite{yang2019learning,yang2020learning,Shen_2021_CVPR,Nguyen_2021_CVPR}, clustering performance is optimized via Binary Cross Entropy (BCE) or Softmax function. However, none of these loss functions reflect clustering metrics such as Fowlkes-Mallows (denoted as $F_P$ in \cite{yang2019learning,yang2020learning,Shen_2021_CVPR,Nguyen_2021_CVPR}).
To achieve better performance, we introduce a novel loss function named Fowlkes-Mallows Loss. To the best of our knowledge, no study has been conducted on adopting a supervised loss for the unsupervised clustering problem, so this is the first time Fowlkes-Mallows-based loss has been introduced in deep learning.
Given an input cluster $\mathcal{N}(\mathbf{x}_i)$, the corresponding output of the network is $q_i \in [0, 1]$ 
estimating a probability for the $i$-th positive sample having the same cluster ID as the centroid or not.
The Fowlkes-Mallows Index (FMI) is measured as, 
\begin{equation}
    f(q_i, \hat{q}_i) = \frac{TP}{\sqrt{(TP + FN) (TP + FP)}}
\end{equation}
Since $0 \leq FMI \leq 1$, to maximize FMI, we minimize $1 - f(q_i, \hat{q}_i)$.

\noindent \textbf{Theoretical Analysis}. Since $TP, FN, \text{and } FP \geq 0$, following Cauchy–Schwarz inequality, we have:
\begin{equation}
    \sqrt{(TP + FN) (TP + FP)} \leq \frac{1}{2}(2TP + FN + FP)
\end{equation}
The \textbf{upper bound} of $1 - f(q_i, \hat{q}_i)$ is derived as follows:
\begin{equation}
\begin{split}
    1 - f(q_i, \hat{q}_i) &\leq 1 - \frac{2TP}{2TP + FN + FP}  \\
                &= \frac{FN + FP}{2TP + FN + FP} = \mathcal{L}^{FMI}
\end{split}
\end{equation}
From this observation, we can maximize the FMI by minimizing the function $\mathcal{L}^{FMI}$. In practice, the implementation of $\mathcal{L}^{FMI}$ is presented in Algorithm \ref{algo:new_loss}. The converging point is when $FP = FN = 0$, and all $TP$ samples are estimated correctly. Consequently, we adopt $\mathcal{L}^{FMI}$ for $\mathcal{L}$ in Eqn. \eqref{eq:upper_bound_objective_function} to maximize the clustering performance.
\begin{algorithm}[tb]
   \caption{Pseudocode of Fowlkes-Mallows Loss}
   \label{algo:new_loss}
    \definecolor{codeblue}{rgb}{0.25,0.5,0.5}
    \lstset{
      basicstyle=\fontsize{7.2pt}{7.2pt}\ttfamily\bfseries,
      commentstyle=\fontsize{7.2pt}{7.2pt}\color{codeblue},
      keywordstyle=\fontsize{7.2pt}{7.2pt},
    }
\begin{lstlisting}[language=python]
def fowlkes_mallows_loss(outputs, targets, threshold):
    # outputs: A list of predicted elements
    # targets: A list of ground truths/targets
    # threshold: threshold for outputs binarization
    if threshold is not None:
        outputs = (outputs > threshold).float()
    tp = torch.sum(outputs * targets)
    fp = torch.sum(outputs) - tp
    fn = torch.sum(targets) - tp
    return (fp + fn) / (2 * tp + fn + fp)
\end{lstlisting}
\end{algorithm}

\begin{figure*}[t]
    \includegraphics[width=0.9\textwidth]{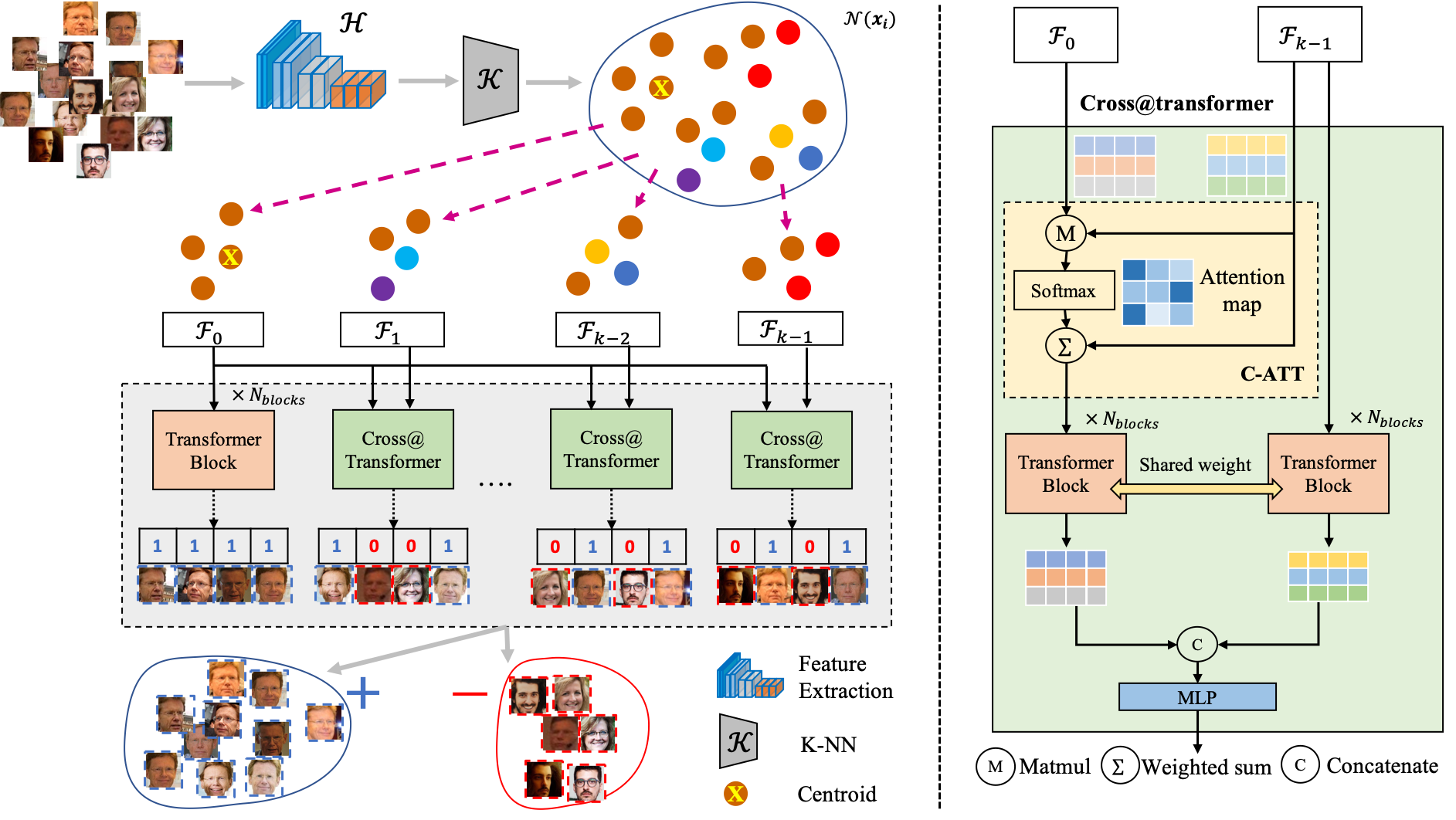}
    \centering
    \caption{Intraformer framework. The images are fed to the feature extractor $\mathcal{H}$ to generate deep feature vectors. For each sample $\mathbf{x}_i$, K-nearest neighbor algorithm $\mathcal{K}$ is employed to construct the cluster $\mathcal{N}(\mathbf{x}_i)$. This cluster is decomposed into $k$ sub-clusters named $\mathcal{F}_0, \dots \mathcal{F}_{k-1}$ and fed into Intraformer concurrently. Only $\mathcal{F}_0$ is passed to the Transformer block while the rest of the sub-clusters go through Cross@Transformer. In this block, C-ATT is employed to explore the global information between two clusters $\mathcal{F}_0$ and $\mathcal{F}_{k-1}$. %
    }
    \label{fig:intraformerarch}
\end{figure*}

\subsection{Fairness Penalty with Clustering Purity Loss}
We aim to solve the second term $(O2)$ of Eqn in parallel to accuracy. \eqref{eq:upper_bound_objective_function} to maximize the 
fairness across demographic groups. As annotations for protected attributes are unknown, a direct grouping of samples for optimization is not feasible. Thus, we propose to address the fairness aspect at the instance level. Promotion of the consistency of clustering purity $\gamma_i$ for all clusters (i.e., enforcing all clusters to have the same purity value) can produce an equivalent solution to minimize the group discrepancy $\Delta_{DP}(\Phi)$ of Eqn. \eqref{eq:fairness}. Formally, $(O2)$ can be reformulated as.
\begin{equation}
    \sum_{i=1}^p \mathbb{E}_{\mathbf{x}_i \sim p(g_i)}\left[\mathcal{L}^{\textit{fair}}(\cdot) \right] \equiv \mathbb{E}_{\mathbf{x}_i \sim p(\mathcal{X})}\left[\mathcal{L}^{\textit{fair}} (\cdot) \right]
\end{equation}
where $\mathcal{L}^{\textit{fair}}(\cdot)$ penalizes the discrepancy between $\gamma_i$ of the $i$-th cluster an a reference $\gamma_f$ within a batch as in Eqn. \eqref{eq:L_fair}.
\begin{equation} \label{eq:L_fair}
    \mathcal{L}^{fair} = \frac{1}{B}\sum \vert\gamma_i - \gamma_{f}\vert
\end{equation}
Here, $B$ is the batch size, 
and $\gamma_{f}$ is the \textit{fairness point} that we want all clusters' purity to converge to. 

\noindent
\textbf{Selection of $\gamma_{f}$}. If $\gamma_{f}$ is too small, the $\mathcal{L}^{fair}$ is easy to optimize, but the lower value of $\gamma_{f}$ will decrease the performance overall. If $\gamma_{f}$ is too large, it will be hard to find the converging points of $\mathcal{L}^{fair}$. In order to solve this problem, we define $\gamma_{f}$ as a \textit{flexible} value and adaptive according to the current performance of the model during the learning stage.
We select $\gamma_{f} = \frac{1}{B} \sum_{i}^{B-1} \gamma_i$ as the average value of $\gamma_{i}$ within a mini-batch. Notice that other selections (such as median and percentile) can still be applicable for $\mathcal{L}^{fair}$.

\noindent
\textbf{Optimizing with $\mathcal{L}^{fair}$}. Initially, the network is warmed up with only $\mathcal{L}^{FMI}$ loss (i.e. $\lambda = 0$). Then, as the network achieves considerable accuracy after warming up, $\lambda$ is gradually increased to penalize the clustering fairness and enhance consistency across clusters.

\noindent
\textbf{The choice of $\gamma_{i}$}. With  $\mathcal{L}^{fair}$ form, besides cluster purity $\gamma_{i}$, other differential metrics can be flexibly selected.
We leave the investigation of these metrics for future work.

\noindent
In the next section, we propose a novel architecture that enhances the performance of hard clusters belonging to minority groups. 
By doing so, $\mathcal{L}^{fair}$ will converge faster.

\section{Intraformer Architecture}
\label{sec:proposed_method}

\begin{figure}[t]
    \centering
\includegraphics[width=1.0 \columnwidth]{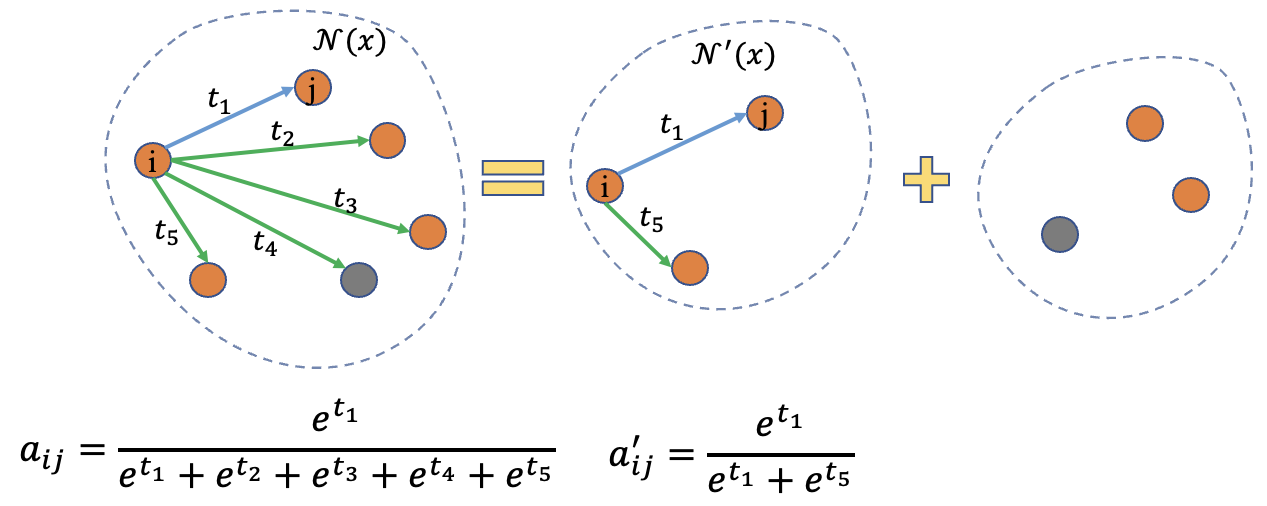}
    \caption{$\mathcal{N}(\mathbf{x})$ is decomposed to a smaller cluster $\mathcal{N}^{'}(\mathbf{x})$ which leads to higher attention score between sample $i^{th}$ and $j^{th}$.}
    \label{fig:revisiting_self_att}
\end{figure}

As clusters $\mathcal{N}(\mathbf{x}_i)$ of the minor group contain a large number of noisy/negative samples (especially when $k$ is large),  $\mathcal{M}$ easily failed to recognize them. 
Therefore, rather than learning directly from all samples of the large cluster at once,
we propose first to decompose $\mathcal{N}(\mathbf{x}_i)$ into {$k$} sub-clusters. Each sub-cluster $C^m_i \subset \mathcal{N}(\mathbf{x}_i)$ has {$s = \frac{n}{k}$ samples where $\mathcal{N}(\mathbf{x}_i) = \bigcup_{m=0}^{k-1} C_i^m$ and $C^0_i \bigcap C^1_i \dots \bigcap C^{k-1}_i = \emptyset$, with $0 \leq m \leq k-1$}. Because $\mathcal{N}(\mathbf{x}_i)$ is an unstructured set, 
two constraints are defined for the sub-clusters to guarantee order consistency, as in Eqn. \eqref{eq:e_importance}.
\begin{equation}
	\begin{cases}
	sim(\mathbf{x}_0, \mathbf{x}_j) \ge sim(\mathbf{x}_0, \mathbf{x}_{j+1}), \\
	sim(\mathbf{x}_0, \mathbf{x}_{s-1}) \ge sim(\mathbf{x}_0, \mathbf{x}^{+}_0),
	\end{cases}
    \label{eq:e_importance}
\end{equation}
where $\mathbf{x}_0, \mathbf{x}_j, \mathbf{x}_{j+1},  \mathbf{x}_{s-1} \in C^k_i$ 
are in a same sub-cluster, while $\mathbf{x}^{+}_0 \in C^{k+1}_i$ denotes the sample in the next sub-clusters. 
Correlations between samples \textit{within a sub-cluster} and \textit{across sub-clusters} are then exploited. 
Moreover, as the size of sub-clusters is kept equal, the balance between positive (hard) and noisy samples is effectively maintained. %

\noindent
\textbf{Intraformer Architecture.} \label{subsec:framework}
Following constraints in Eqn. \eqref{eq:e_importance}, the centroid $\mathbf{x}_i$ belongs to $C^0_i$. 
We define the feature of $C^m_i$ as the concatenated features $\mathcal{F}_m \in \mathbb{R}^{s \times d}$ of all samples in $C^m_i$, i.e. $\mathcal{F}_m \triangleq \text{concat}(\mathbf{x}_0, \mathbf{x}_1, \dots \mathbf{x}_{s-1})$ where $\mathbf{x}_0, \dots \mathbf{x}_{s-1} \in C^m_i$.
Taking $\{\mathcal{F}_m\}_{m=0}^{k-1}$ as input, the proposed Intraformer architecture 
is presented as in Fig. \ref{fig:intraformerarch}. It consists of a Transformer block with a self-attention (S-ATT) and $(k-1)$ Cross@transformers which includes both S-ATT and Cross Attention (C-ATT) Mechanisms. 
Notice that as $C^0_i$ includes the centroid of $\mathcal{N}(\mathbf{x}_i)$, only S-ATT is needed to explore correlations between samples of $C^0_i$. Therefore, we use Transformer blocks for $C^0_i$. For other sub-clusters, they are fed into Cross@transformer blocks with C-ATT mechanism for measuring correlations to $\mathbf{x}_i$.

\begin{figure}[t]
    \centering
    \includegraphics[width=1.0\columnwidth]{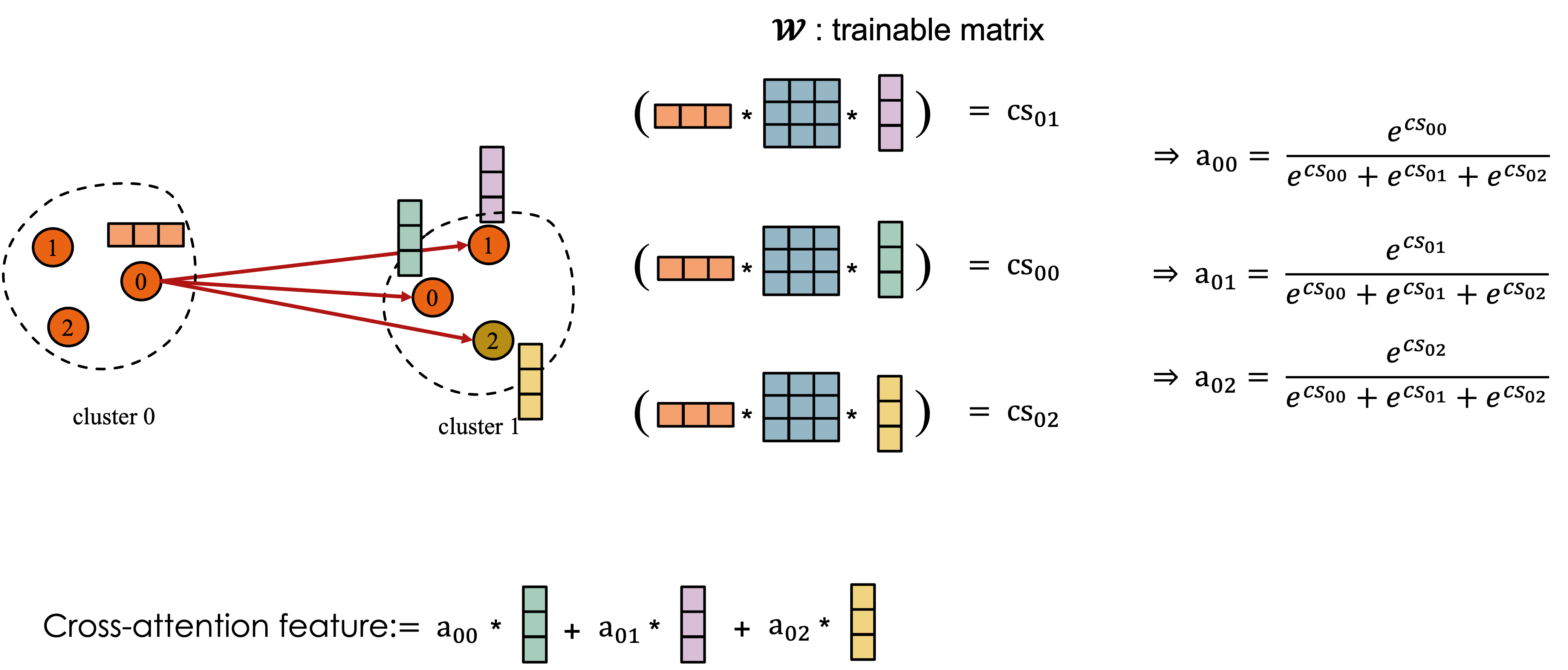}
    \caption{The correlation score is measured by multiple matrix multiplications. The \textcolor{black}{orange vector} and blue \textcolor{black}{green vector} are represented for the feature of sample-0 in \textit{cluster 0}  and sample-1 in  \textit{cluster 1}, respectively. These features are fixed. The \textcolor{black}{blue matrix} are the learnable correlation matrix.
    }
    \label{fig:correlation_score}
\end{figure}

\noindent
\textbf{Decomposing Cluster Benefits Self-Attention Mechanism.}  \label{subsec:self_att_revisit}
Given $i^{th}$ and $j^{th}$ samples of a cluster $\mathcal{N}(\mathbf{x})$, their attention score is measured as $a_{ij} = \frac{e^{\frac{1}{\sqrt{d'}}\mathbf{Q}_i \mathbf{K}_j^T}}{\sum_{m=1}^{n}{e^{\frac{1}{\sqrt{d'}}\mathbf{Q}_i \mathbf{K}_m^T}}}$, 
where $\mathbf{Q}$ and $\mathbf{K}$ denote the \textit{query} and \textit{key} in a transformer block, respectively. This score 
also indicates the importance of their correlation against correlations of $i^{th}$ and other samples in $\mathcal{N}(\mathbf{x})$.
Let  
$\mathcal{N}^{'}(\mathbf{x})$ be a smaller sub-cluster of $\mathcal{N}(\mathbf{x})$ where  $i, j < n' = \vert \mathcal{N}^{'}(\mathbf{x})\vert < n$. In addition, 
as $e^{\frac{1}{\sqrt{d'}}\mathbf{Q}_i \mathbf{K}_m^T} > 0$, 
$a_{ij}$ can be rewritten as follows. %
\begin{equation}
    \label{eq:att_score_2}
\begin{split}
    a_{ij}
    &= \frac{e^{\frac{1}{\sqrt{d'}}\mathbf{Q}_i \mathbf{K}_j^T}}{\sum_{m=1}^{n'}{e^{\frac{1}{\sqrt{d'}}\mathbf{Q}_i \mathbf{K}_m^T}} + \sum_{m=n'}^{n}{e^{\frac{1}{\sqrt{d'}}\mathbf{Q}_i \mathbf{K}_m^T}}} \\
    &< \frac{e^{\frac{1}{\sqrt{d'}}\mathbf{Q}_i \mathbf{K}_j^T}}{\sum_{m=1}^{n'}{e^{\frac{1}{\sqrt{d'}}\mathbf{Q}_i \mathbf{K}_m^T}}} = {a_{ij}}'
\end{split}
\end{equation}
where ${a_{ij}}'$ denotes the attention score of the $i^{th}$ and $j^{th}$ samples in $\mathcal{N}^{'}(\mathbf{x})$. 
Therefore, the number of samples in a cluster greatly impacts their attention scores. The more samples are in a cluster, the weaker correlations 
between samples can be extracted. 
Thus, decomposing $\mathcal{N}(\mathbf{x})$ into multiple smaller sub-clusters will benefit the attention mechanism and help $\mathcal{M}$ focus on the hard samples. %
Fig. \ref{fig:revisiting_self_att} where $t \sim \frac{1}{\sqrt{d'}}\mathbf{Q} \mathbf{K}^T $ briefly demonstrates the cluster decomposition and attention benefits.

\noindent
\textbf{Cross Sub-cluster Attention Mechanism.}
\label{subsec:crossattention}
As the centroid is only assigned to the first sub-clusters $C^0_i$ of $\mathcal{N}(\mathbf{x}_i)$, dividing into non-overlapped sub-clusters leads to an issue of ignoring the interactions between samples in $k-1$ sub-clusters ($C^1_i, \dots C^{k-1}_i$) and its centroid.
Given two samples $\mathbf{x}_i$  and $\mathbf{x}_j$ of two different sub-clusters, 
Since $\mathbf{x}_i$ and $\mathbf{x}_j$ are fixed, two learnable matrices $\mathbf{W}^i \in \mathbb{R}^{d \times d'}$ and $\mathbf{W}^j \in \mathbb{R}^{d \times d'}$ are presented to transform these features to a new $d'$ dimensional hyperspace where their correlations can be computed 
as in Eqn. \eqref{eq:cosine_hyperspace}.
\begin{equation} \label{eq:cosine_hyperspace}
\begin{split}
    cs_{ij} = (\mathbf{x}_{i}\mathbf{W}^i)(\mathbf{x}_{j}\mathbf{W}^j)^T = 
    \mathbf{x}_{i}\mathbf{W}{\mathbf{x}_{j}}^T
\end{split}
\end{equation}
where $\mathbf{W} = \mathbf{W}^i{\mathbf{W}^j}^T$.
The illustration of how to measure this correlation score is shown in Fig. \ref{fig:correlation_score}.
In Eqn. \eqref{eq:cosine_hyperspace}, the number of parameters in the attention module can be reduced by setting $d' = d$. 
Therefore, instead of updating $\mathbf{W}^i$ and $\mathbf{W}^j$, we need only one trainable matrix $\mathbf{W} \in \mathbb{R}^{d \times d}$ called correlation matrix to reduce complexity but still keeps the performance of the attention module. 
As the cosine similarity score is unable to represent the importance of a sample in the cluster, a new attention score is proposed to measure how an out-of-cluster sample  $\mathbf{x}_o \not\in C_i^k$ looks at a sample $\mathbf{x}_j$ in the $k$-th sub-cluster $C_i^k$ 
as follows.
\begin{equation} \label{eq:newscore}
    a(\mathbf{x}_o, \mathbf{x}_j) = \frac{\exp(\mathbf{x}_o \mathbf{W} \mathbf{x}_j^T)}{\sum_{j'=0}^{s-1} \exp(\mathbf{x}_o \mathbf{W} {\mathbf{x}_{j'}}^T)}
\end{equation}
Where $\{ \mathbf{x}_{j'}\}$ denotes all samples of the $k$-th sub-cluster $C_i^k$; and $s$ is the number of samples in $C_i^k$.
Given this new attention score, the attention-based feature of an out-of-cluster sample looking at a cluster $C_i^k$ can be computed using weighted sum features as in Eqn. \eqref{eq:newscore2}.
\begin{equation} \label{eq:newscore2}
    h(\mathbf{x}_o, C_i^k) = \sum_{j=0}^{s-1} a(\mathbf{x}_o, \mathbf{x}_j) \mathbf{x}_j
\end{equation}

\begin{figure}[t]
    \centering
    \includegraphics[width=1.0\columnwidth]{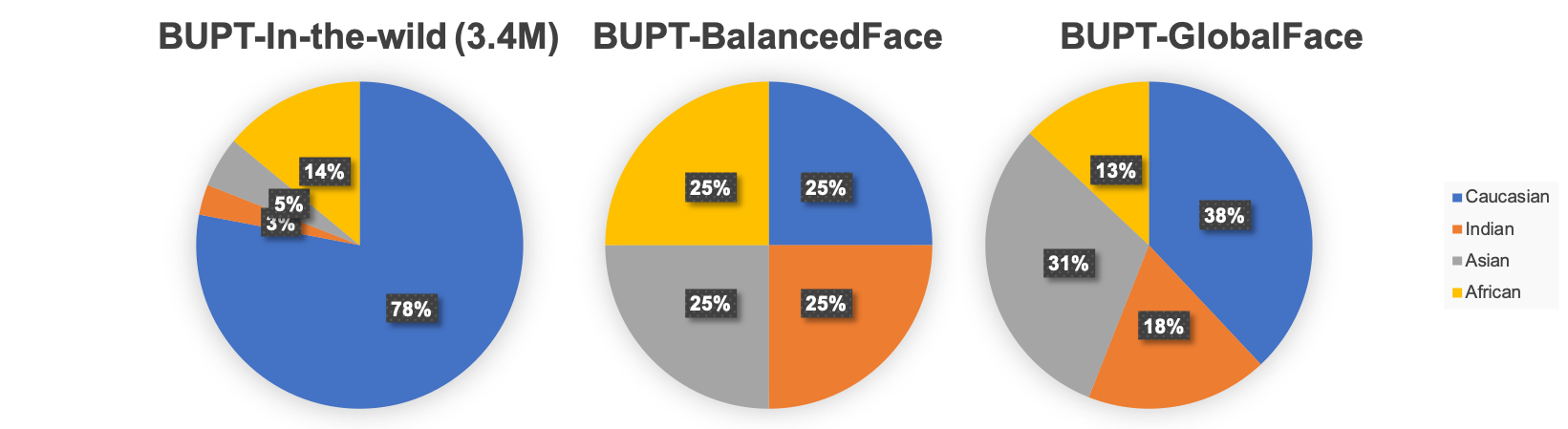}
    \caption{Distributions of ethnicity in BUPT-In-the-wild, BUPT-BalancedFace and BUPT-GlobalFace}
    \label{fig:BUPT_distribution}
\end{figure}

\begin{figure*}[t]
    \centering
    \includegraphics[width=1.0 \textwidth]{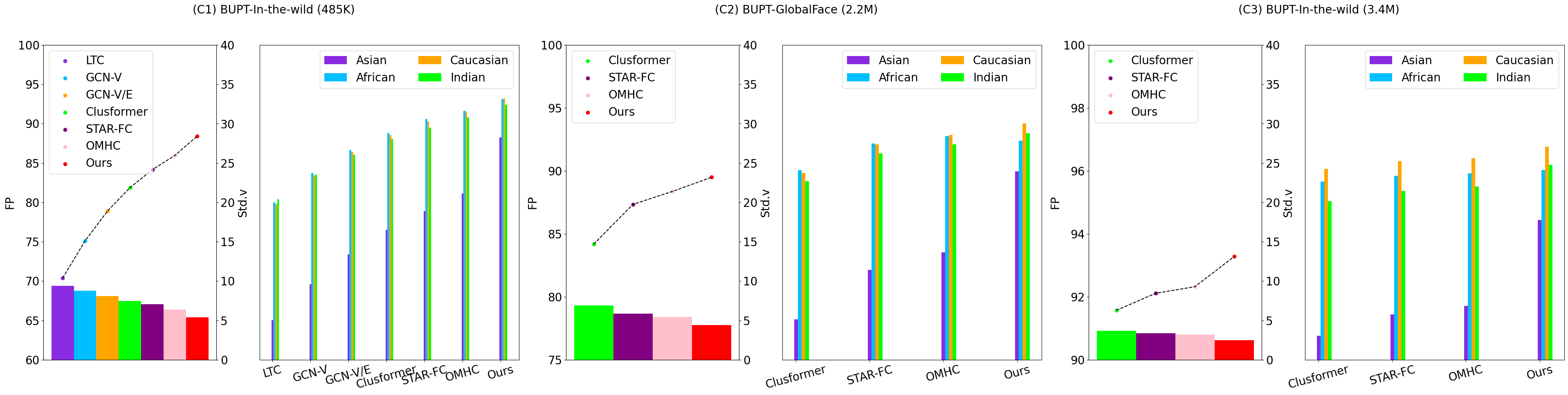}
    \caption{Comparison of fairness \textit{w.r.t} the race of identity on BUPT-BalancedFace}
    \label{fig:balancedface}
\end{figure*}

\begin{figure*}[t]
    \centering
    \includegraphics[width=1.0 \textwidth]{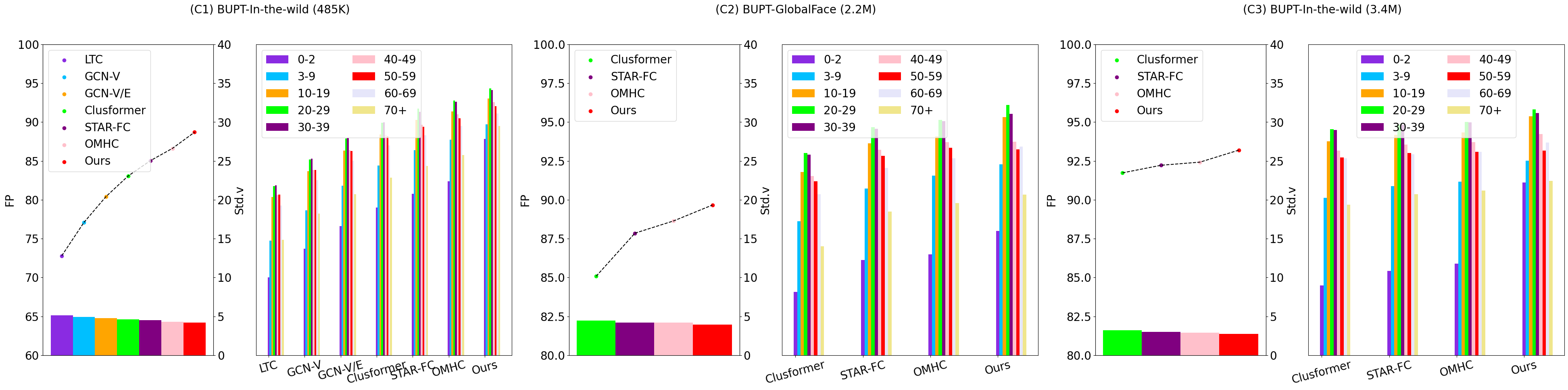}
    \caption{Comparison of fairness \textit{w.r.t} the age of identity on BUPT-BalancedFace}
    \label{fig:age}
\end{figure*}

\begin{figure*}[t]
    \centering
    \includegraphics[width=1.0 \textwidth]{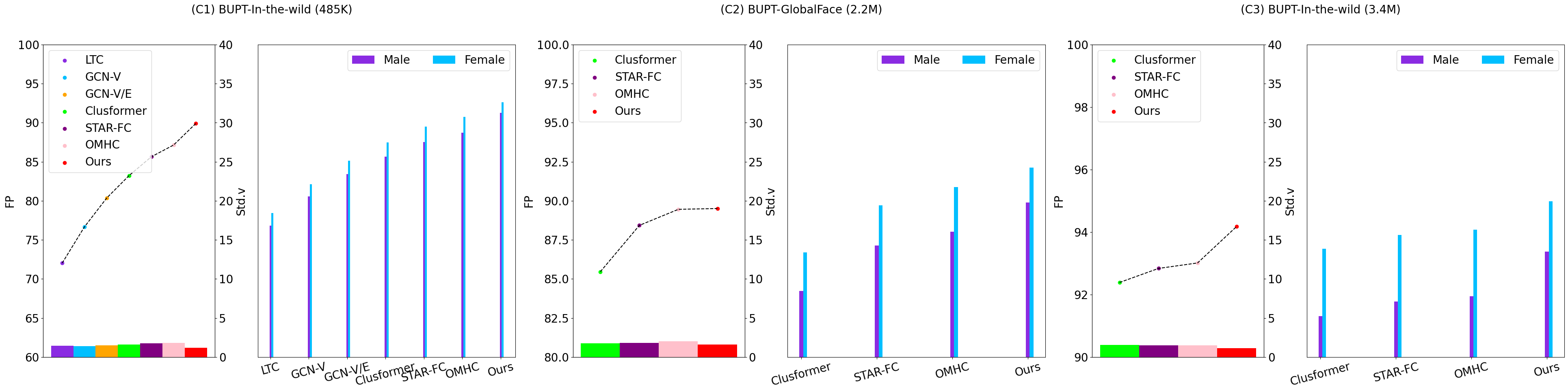}
    \caption{Comparison of fairness \textit{w.r.t} the gender of identity on BUPT-BalancedFace}
    \label{fig:gender}
\end{figure*}

\section{Experiments}
\subsection{Dataset and Metrics} 
\noindent
\textbf{Datasets.}
The visual clustering problem is well defined in \cite{yang2019learning,yang2020learning,Nguyen_2021_CVPR, omhc}. The most common datasets are MS-Celeb-1M, MNIST-Fashion, and Google Landmark. To evaluate the fairness of clustering methods, demographic attributes are needed. Among these datasets, only BUPT-BalancedFace (i.e., different splits of MS-Celeb-1M) provides this information. For this reason,
a subset of MS-Celeb-1M named BUPT-Balancedface \footnote{All data generated or analyzed during this study are included in this published article \cite{Wang_2019_ICCV}} \cite{Wang_2019_ICCV} (i.e., 1.3M images of 28K celebrities with a race-balance of 7K identities per demographic group) is adopted.
We also denote BUPT-In-the-wild (3.4M) as the remaining MS-Celeb-1M after excluding BUPT-Balancedface. This dataset contains 3.4M images of 70K identities 
and is highly biased in terms of the number of identities per race 
as well as 
images per identity.
LTC \cite{yang2019learning}, GCN-V, GCN-E \cite{yang2020learning}, STAR-FC \cite{Shen_2021_CVPR} and Transformer-based methods, i.e: Clusformer \cite{Nguyen_2021_CVPR}, OMHC \cite{omhc} are used as baselines.
Note that LTC, GCN-V, and GCN-V/E generate a large affinity graph without any mechanism to separate this graph into multiple GPUs for training.
Thus, BUPT-In-the-wild cannot be used for their training process.
We propose to decompose BUPT-In-the-wild into smaller parts:

\noindent
\textbf{(C1) BUPT-In-the-wild (485K)}. BUPT-In-the-wild (3.4M) is randomly split into seven parts, each consisting of 485K samples of 7K identities. The race distributions  
are similar to 
BUPT-In-the-wild (3.4M) (see Fig \ref{fig:BUPT_distribution}). These parts are sufficient to train LTC, GCN-V, GCN-E, STAR-FC, Clusformer, OMHC, and Intraformer. 

\noindent
\textbf{(C2) BUPT-GlobalFace (2.2M)}. This subset provided by \cite{Wang_2019_ICCV} contains 2M images from 38K celebrities in total. Its racial distribution is approximately the same as the real distribution of the world’s population. Compared to the \textbf{BUPT-In-the-wild (485K)}, this subset is more balanced in 
racial distribution. 
Only STAR-FC, Clusformer, OMHC, and Intraformer are implemented to run on this subset since there are out-of-memory issues with LTC, GCN-V, and GCN-E.

\noindent
\textbf{(C3) BUPT-GlobalFace (3.4M)}. This configuration evaluates fairness with respect to racial attributes against large-scale data of 3.4M images with a highly biased distribution. Similar to \textbf{BUPT-GlobalFace (2.2M)}, Only STAR-FC, Clusformer, OMHC, and Intraformer are included. %

\begin{figure*}[t]
    \centering
    \includegraphics[width=1.0\textwidth]{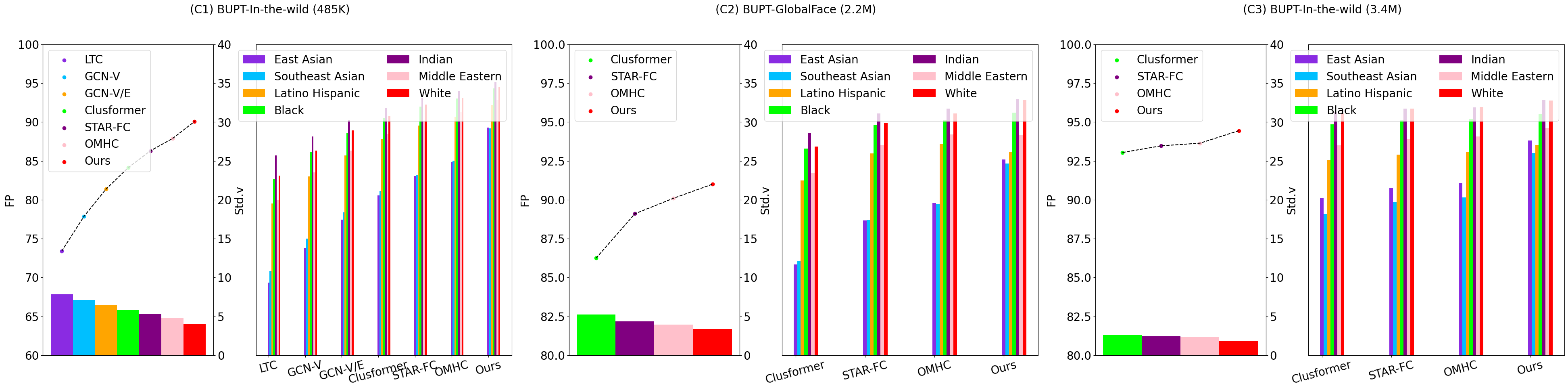}
    \caption{Comparison of fairness \textit{w.r.t} the ethnicity (7 attributes) of identity on BUPT-BalancedFace}
    \label{fig:race_7}
\end{figure*}

\noindent \textbf{Metrics.}
Pairwise F-score ($F_P$), BCubed F-score ($F_B$), and NMI are adopted for evaluation as in previous benchmarking protocols \cite{yang2019learning,yang2020learning,Shen_2021_CVPR,Nguyen_2021_CVPR, omhc}.
To measure the fairness of a method on a demographic, we measure the clustering performance on all attributes of this demographic and estimate the standard deviation \textit{std} among them following previous works \cite{Wang_2019_ICCV, wang2021meta, wang2019skewness}. The lower value of \textit{std}, the better the fairness method is.

\noindent
\textbf{Implementation Details.} %
We get the feature extractor $\mathcal{H}$ by training resnet34 with Arcface \cite{deng2019arcface} loss on BUPT-In-the-wild (3.4) within 20 epochs. The embedding feature is $d=512$ dimension. This model is then adopted to extract latent features followed by a k-NN $\mathcal{K}$ to construct initial clusters. The architecture of the Intraformer contains two main blocks, i.e., Transformer and Cross Transformer.
In each block, transformers are stacked $N_{block}$ times continuously and in each transformers block, self-attention is divided into multiple heads $N_{head}$. 
$k = 4$, $N_{block} = 6$ and $N_{head} = 4$, and $n = 256$. The network is trained in 10 epochs with Adam optimizer \cite{kingma2014adam}. The learning rate starts at 0.0001 and reduces following the cosine annealing \cite{loshchilov2016sgdr}. The loss weight of $(O1)$ equals $(O2)$ in all experiments.

\subsection{Results}

\begin{table*}[!ht]
\centering
\caption{Comparison of fairness \textit{w.r.t} the ethnicity of identity on BUPT-BalancedFace. The results are reported in $F_P$, $F_B$ and $NMI$ metrics. Asi: \textit{Asian}, Afr: \textit{African}, Cau: \textit{Caucasian}, Ind: \textit{Indian}}
\setlength{\tabcolsep}{3pt}
\resizebox{1.\textwidth}{!}{
    \begin{tabular}{|l|cccccc|cccccc|cccccc|}
    \hline
    \multirow{2}{*}{Method} & \multicolumn{6}{c|}{$F_P$} & \multicolumn{6}{c|}{$F_B$} & \multicolumn{6}{c|}{$NMI$} \\
                            & Asi & Afr & Cau & Ind & Mean & STD & Asi & Afr & Cau & Ind & Mean & STD & Asi & Afr & Cau & Ind & Mean & STD  \\

\hline 
\multicolumn{19}{|c|}{(C1) BUPT-In-the-wild (485K)} \\ 
\hline 
LTC~\cite{yang2019learning} & 56.3 & 74.96 & 74.73 & 75.49 & 70.37 & 9.39 & 59.38 & 74.73 & 74.39 & 75.85 & 71.09 & 7.83 & 91.52 & 94.07 & 93.92 & 94.25 & 93.44 & 1.29 \\
GCN-V~\cite{yang2020learning} & 61.98 & 79.68 & 79.3 & 79.43 & 75.1 & 8.75 & 64.59 & 79.2 & 78.76 & 79.35 & 75.48 & 7.26 & 92.48 & 95.05 & 94.91 & 95.06 & 94.38 & 1.27 \\
GCN-V/E~\cite{yang2020learning} & 66.74 & 83.28 & 82.99 & 82.54 & 78.89 & 8.1 & 68.87 & 82.59 & 82.24 & 82.14 & 78.96 & 6.73 & 93.27 & 95.8 & 95.69 & 95.69 & 95.11 & 1.23 \\
Clusformer~\cite{Nguyen_2021_CVPR} & 70.64 & 86.05 & 85.7 & 85.03 & 81.86 & 7.49 & 72.38 & 85.17 & 84.81 & 84.32 & 81.67 & 6.2 & 93.94 & 96.36 & 96.28 & 96.19 & 95.69 & 1.17 \\
STAR-FC~\cite{Shen_2021_CVPR} & 73.62 & 88.22 & 87.87 & 86.89 & 84.15 & 7.04 & 75.18 & 86.91 & 86.68 & 85.88 & 83.66 & 5.68 & 94.26 & 96.85 & 96.67 & 96.59 & 96.09 & 1.22 \\
OMHC~\cite{omhc} & 76.43 & 89.54 & 89.42 & 88.46 & 85.96 & 6.37 & 77.49 & 88.42 & 88.31 & 87.35 & 85.39 & 5.29 & 94.92 & 97.1 & 97.08 & 96.88 & 96.5 & 1.05 \\
\textbf{Ours} & \textbf{80.32} & \textbf{91.4} & \textbf{91.45} & \textbf{90.48} & \textbf{88.41} & \textbf{5.41} & \textbf{80.84} & \textbf{90.02} & \textbf{90.17} & \textbf{89.04} & \textbf{87.52} & \textbf{4.48} & \textbf{95.52} & \textbf{97.44} & \textbf{97.5} & \textbf{97.26} & \textbf{96.93} & \textbf{0.95} \\

\hline 
\multicolumn{19}{|c|}{(C2) BUPT-GlobalFace (2.2M)} \\ 
\hline 
Clusformer~\cite{Nguyen_2021_CVPR} & 73.85 & 88.07 & 87.79 & 87.02 & 84.18 & 6.9 & 75.23 & 87.06 & 86.8 & 86.08 & 83.79 & 5.72 & 94.48 & 96.79 & 96.73 & 96.59 & 96.15 & 1.11 \\
STAR-FC~\cite{Shen_2021_CVPR} & 78.56 & 90.61 & 90.53 & 89.65 & 87.34 & 5.87 & 79.35 & 89.34 & 89.34 & 88.36 & 86.6 & 4.85 & 95.26 & 97.3 & 97.31 & 97.11 & 96.75 & 0.99 \\
OMHC~\cite{omhc} & 80.22 & \textbf{91.29} & 91.4 & 90.54 & 88.36 & 5.44 & 80.73 & 90.19 & 90.06 & 88.81 & 87.45 & 4.52 & 95.31 & 97.34 & 97.71 & 97.06 & 96.85 & 1.06 \\
\textbf{Ours} & \textbf{82.96} & 90.89 & \textbf{92.51} & \textbf{91.59} & \textbf{89.49} & \textbf{4.4} & \textbf{83.92} & \textbf{91.19} & \textbf{91.9} & \textbf{90.33} & \textbf{89.34} & \textbf{3.67} & \textbf{95.93} & \textbf{97.5} & \textbf{97.75} & \textbf{97.44} & \textbf{97.16} & \textbf{0.83} \\

\hline 
\multicolumn{19}{|c|}{(C3) BUPT-In-the-wild (3.4M)} \\ 
\hline 
Clusformer~\cite{Nguyen_2021_CVPR} & 86.14 & 93.48 & 94.09 & 92.55 & 91.57 & 3.67 & 86.53 & 92.89 & 93.48 & 91.61 & 91.13 & 3.16 & 96.5 & 98.11 & 98.26 & 97.81 & 97.67 & 0.8 \\
STAR-FC~\cite{Shen_2021_CVPR} & 87.15 & 93.76 & 94.46 & 93.05 & 92.11 & 3.35 & 87.52 & 93.26 & 93.93 & 92.18 & 91.72 & 2.89 & 96.74 & 98.21 & 98.38 & 97.96 & 97.82 & 0.74 \\
OMHC~\cite{omhc} & 87.57 & 93.87 & 94.6 & 93.25 & 92.32 & 3.22 & 87.93 & 93.4 & 94.12 & 92.41 & 91.97 & 2.78 & 96.84 & 98.25 & 98.43 & 98.02 & 97.89 & 0.72 \\
\textbf{Ours} & \textbf{89.65} & \textbf{94.04} & \textbf{95.15} & \textbf{94.28} & \textbf{93.28} & \textbf{2.47} & \textbf{90.04} & \textbf{93.91} & \textbf{94.87} & \textbf{93.5} & \textbf{93.08} & \textbf{2.11} & \textbf{97.42} & \textbf{98.36} & \textbf{98.63} & \textbf{98.31} & \textbf{98.18} & \textbf{0.53} \\

\hline
\end{tabular}
}
\label{tab:enthic}
\end{table*}

\subsubsection{Fairness on ethnicity}

Fig. \ref{fig:balancedface} shows the performance of clustering algorithms' results on three different training configurations \textbf{(C1)}, \textbf{(C2)}, and \textbf{(C3)}. 
The bars represent the standard deviation of each method and are measured on the right side vertical axis. The scatter plot shows the average $F_P$ score of each and is measured by the vertical axis on the left side of the chart. The second chart in each configuration shows the $F_P$ score of each method in each category. Thus, each method's fairness is determined by the proximity of the bars with greater proximity representing greater fairness.

Overall, our method outperforms others in their respective categories. In configuration \textbf{(C1)}, Intraformer achieved an $F_B$ of 88.41\% which is higher than the SOTA of GCN-based STAR-FC and Transformer-based OMHC by 4\% and 3\%, respectively. In addition, the results demonstrate that our method is the \textit{fairest} by scoring 5.41\% on \textit{std}, lower than the second-best OMHC by 1.1\%. 
In configuration \textbf{(C2)}, when the number of training samples is 4.5 times larger than \textbf{(C1)} and the distribution of demographic attributes is more balanced, improved performance is expected. All methods get higher scores than \textbf{(C1)} in all categories, and Intraformer stills holds the best performance in terms of accuracy (89.49\%) and fairness (4.40\%) among Clusformer, STAR-FC, and OMHC. In configuration \textbf{(C3)}, where the number of samples is 3.4M images, but the distribution of demographic attributes is highly biased, it is interesting to note that the performance is higher than \textbf{(C2)}. Once again, the best results are achieved by Intraformer when trained on the largest dataset, with $F_P$ of 93.28\%  and \textit{std} of 2.47\%. Obtaining \textbf{(C2)} is difficult since demographic annotations are required. Instead, \textbf{(C3)} is preferred as it is easy to collect.
Configuration \textbf{(C3)} illustrates that our method efficiently handles highly biased training databases and achieves better fairness and clustering performance. We observe similar results on the $F_B$ and NMI metrics as well.

\begin{table*}[!ht]
\centering
\caption{Comparison of fairness \textit{w.r.t} the pseudo gender of identity on BUPT-BalancedFace. The results are reported in $F_P$, $F_B$ and $NMI$ metrics}
\setlength{\tabcolsep}{8pt}
    \begin{tabular}{|l|cccc|cccc|cccc|}
    \hline
    \multirow{2}{*}{Method} & \multicolumn{4}{c|}{$F_P$} & \multicolumn{4}{c|}{$F_B$} & \multicolumn{4}{c|}{$NMI$}
    \\ 
                            & Male & Female & Mean & STD & Male & Female & Mean & STD & Male & Female & Mean & STD  \\

\hline 
\multicolumn{13}{|c|}{(C1) BUPT-In-the-wild (485K)} \\ 
\hline 
LTC~\cite{yang2019learning} & 71.0 & 73.03 & 72.02 & 1.44 & 71.94 & 72.22 & 72.08 & 0.2 & 94.39 & 93.88 & 94.14 & 0.36 \\
GCN-V~\cite{yang2020learning} & 75.67 & 77.65 & 76.66 & 1.4 & 76.17 & 76.59 & 76.38 & 0.3 & 95.17 & 94.82 & 95.0 & 0.25 \\
GCN-V/E~\cite{yang2020learning} & 79.29 & 81.43 & 80.36 & 1.51 & 79.43 & 80.18 & 79.81 & 0.53 & 95.77 & 95.58 & 95.68 & 0.13 \\
Clusformer~\cite{Nguyen_2021_CVPR} & 82.09 & 84.36 & 83.22 & 1.61 & 81.95 & 82.97 & 82.46 & 0.72 & 96.24 & 96.19 & 96.22 & \textbf{0.04} \\
STAR-FC~\cite{Shen_2021_CVPR} & 84.39 & 86.87 & 85.63 & 1.75 & 84.09 & 85.11 & 84.6 & 0.72 & 96.5 & 96.45 & 96.47 & 0.04 \\
OMHC~\cite{omhc} & 85.87 & 88.41 & 87.14 & 1.8 & 85.32 & 86.85 & 86.08 & 1.08 & 96.88 & 97.02 & 96.95 & 0.1 \\
\textbf{Ours} & \textbf{89.08} & \textbf{90.74} & \textbf{89.91} & \textbf{1.17} & \textbf{89.2} & \textbf{89.03} & \textbf{89.12} & \textbf{0.12} & \textbf{97.21} & \textbf{97.48} & \textbf{97.34} & 0.19 \\

\hline 
\multicolumn{13}{|c|}{(C2) BUPT-GlobalFace (2.2M)} \\ 
\hline 
Clusformer~\cite{Nguyen_2021_CVPR} & 84.22 & 86.7 & 85.46 & 1.75 & 83.86 & 85.21 & 84.54 & 0.95 & 96.61 & 96.67 & 96.64 & \textbf{0.04} \\
STAR-FC~\cite{Shen_2021_CVPR} & 87.14 & 89.71 & 88.42 & 1.82 & 86.4 & 88.07 & 87.24 & 1.18 & 97.08 & 97.28 & 97.18 & 0.14 \\
OMHC~\cite{omhc} & 88.02 & 90.88 & 89.45 & 2.02 & 87.39 & 88.89 & 88.14 & 1.06 & 96.97 & 97.71 & 97.34 & 0.53 \\
\textbf{Ours} & \textbf{89.88} & \textbf{92.13} & \textbf{89.5} & \textbf{1.59} & \textbf{89.78} & \textbf{90.89} & \textbf{90.34} & \textbf{0.78} & \textbf{97.6} & \textbf{97.74} & \textbf{97.67} & 0.1 \\

\hline 
\multicolumn{13}{|c|}{(C3) BUPT-In-the-wild (3.4M)} \\ 
\hline 
Clusformer~\cite{Nguyen_2021_CVPR} & 91.31 & 93.46 & 92.38 & 1.52 & 90.95 & 92.26 & 91.6 & 0.93 & 97.92 & 98.1 & 98.01 & 0.13 \\
STAR-FC~\cite{Shen_2021_CVPR} & 91.77 & 93.9 & 92.84 & 1.51 & 91.51 & 92.82 & 92.16 & 0.93 & 98.05 & 98.24 & 98.14 & 0.13 \\
OMHC~\cite{omhc} & 91.94 & 94.08 & 93.01 & 1.51 & 91.74 & 93.06 & 92.4 & 0.93 & 98.1 & 98.29 & 98.2 & 0.13 \\
\textbf{Ours} & \textbf{93.37} & \textbf{94.98} & \textbf{94.18} & \textbf{1.14} & \textbf{93.73} & \textbf{94.16} & \textbf{93.94} & \textbf{0.3} & \textbf{98.42} & \textbf{98.59} & \textbf{98.5} & \textbf{0.12} \\

\hline
    \end{tabular}
\label{tab:gender}
\end{table*}

\begin{table*}[!ht]
\centering
\caption{Comparison of fairness \textit{w.r.t} the pseudo ethnicity of identity on BUPT-BalancedFace. The results are reported in $F_P$, $F_B$ and $NMI$ metrics}
\setlength{\tabcolsep}{3pt}
\resizebox{1.\textwidth}{!}{
    \begin{tabular}{|l|cccccc|cccccc|cccccc|}
    \hline
    \multirow{2}{*}{Method} & \multicolumn{6}{c|}{$F_P$} & \multicolumn{6}{c|}{$F_B$} & \multicolumn{6}{c|}{$NMI$}  \\ 
    \cline{2-19}
    & Asian & Indian & Black & White & Mean & STD & Asian & Indian & Black & White & Mean & STD & Asian & Indian & Black & White & Mean & STD \\

\hline 
\multicolumn{19}{|c|}{(C1) BUPT-In-the-wild (485K)} \\ 
\hline 
LTC~\cite{yang2019learning} & 60.15 & 77.26 & 75.81 & 76.47 & 72.42 & 8.2 & 64.49 & 76.9 & 75.8 & 75.86 & 73.26 & 5.87 & 92.83 & 95.28 & 94.85 & 94.84 & 94.45 & 1.1 \\
GCN-V~\cite{yang2020learning} & 65.62 & 81.25 & 80.33 & 80.81 & 77.0 & 7.6 & 69.26 & 80.57 & 79.8 & 79.89 & 77.38 & 5.42 & 93.7 & 96.01 & 95.66 & 95.67 & 95.26 & 1.05 \\
GCN-V/E~\cite{yang2020learning} & 70.18 & 84.46 & 83.64 & 84.22 & 80.63 & 6.97 & 73.18 & 83.51 & 82.8 & 83.05 & 80.64 & 4.98 & 94.43 & 96.59 & 96.27 & 96.32 & 95.9 & 0.99 \\
Clusformer~\cite{Nguyen_2021_CVPR} & 73.95 & 86.89 & 86.27 & 86.67 & 83.45 & 6.34 & 76.39 & 85.75 & 85.16 & 85.36 & 83.17 & 4.52 & 95.04 & 97.04 & 96.74 & 96.8 & 96.41 & 0.92 \\
STAR-FC~\cite{Shen_2021_CVPR} & 77.1 & 89.04 & 88.22 & 88.42 & 85.7 & 5.74 & 78.95 & 87.68 & 87.0 & 87.15 & 85.2 & 4.17 & 95.75 & 97.36 & 97.22 & 97.19 & 96.88 & 0.76 \\
OMHC~\cite{omhc} & 79.33 & 90.22 & 89.64 & 90.06 & 87.31 & 5.33 & 80.96 & 88.83 & 88.2 & 88.56 & 86.64 & 3.79 & 95.9 & 97.65 & 97.36 & 97.47 & 97.1 & 0.81 \\
\textbf{Ours} & \textbf{82.93} & \textbf{92.11} & \textbf{91.34} & \textbf{91.96} & \textbf{89.59} & \textbf{4.45} & \textbf{83.89} & \textbf{90.5} & \textbf{89.64} & \textbf{90.31} & \textbf{88.59} & \textbf{3.15} & \textbf{96.44} & \textbf{97.97} & \textbf{97.64} & \textbf{97.82} & \textbf{97.47} & \textbf{0.7} \\

\hline 
\multicolumn{19}{|c|}{(C2) BUPT-GlobalFace (2.2M)} \\ 
\hline 
Clusformer~\cite{Nguyen_2021_CVPR} & 76.91 & 88.84 & 88.19 & 88.58 & 85.63 & 5.82 & 78.91 & 87.55 & 86.91 & 87.18 & 85.14 & 4.16 & 95.51 & 97.39 & 97.1 & 97.18 & 96.8 & 0.87 \\
STAR-FC~\cite{Shen_2021_CVPR} & 81.3 & 91.31 & 90.61 & 91.13 & 88.59 & 4.87 & 82.6 & 89.8 & 89.03 & 89.56 & 87.75 & 3.45 & 96.21 & 97.84 & 97.53 & 97.67 & 97.31 & 0.75 \\
OMHC~\cite{omhc} & 82.89 & 92.35 & \textbf{91.11} & 91.8 & 89.54 & 4.46 & 84.13 & 90.25 & 89.74 & 90.44 & 88.64 & 3.02 & 96.44 & 97.98 & 97.47 & 97.66 & 97.39 & 0.67 \\
\textbf{Ours} & \textbf{85.37} & \textbf{92.7} & 90.65 & \textbf{92.73} & \textbf{90.36} & \textbf{3.47} & \textbf{86.42} & \textbf{91.74} & \textbf{90.73} & \textbf{91.86} & \textbf{90.19} & \textbf{2.56} & \textbf{96.78} & \textbf{98.1} & \textbf{97.65} & \textbf{98.0} & \textbf{97.63} & \textbf{0.6} \\

\hline 
\multicolumn{19}{|c|}{(C3) BUPT-In-the-wild (3.4M)} \\ 
\hline 
Clusformer~\cite{Nguyen_2021_CVPR} & 88.21 & 94.13 & 93.54 & 94.29 & 92.54 & 2.91 & 88.71 & 93.11 & 92.45 & 93.24 & 91.88 & 2.14 & 97.24 & 98.47 & 98.24 & 98.42 & 98.09 & 0.58 \\
STAR-FC~\cite{Shen_2021_CVPR} & 89.04 & 94.49 & 93.81 & 94.65 & 93.0 & 2.66 & 89.51 & 93.56 & 92.84 & 93.73 & 92.41 & 1.97 & 97.43 & 98.57 & 98.32 & 98.54 & 98.22 & 0.54 \\
OMHC~\cite{omhc} & 89.38 & 94.63 & 93.91 & 94.79 & 93.18 & 2.56 & 89.85 & 93.75 & 92.99 & 93.93 & 92.63 & 1.9 & 97.51 & 98.61 & 98.36 & 98.58 & 98.27 & 0.52 \\
\textbf{Ours} & \textbf{91.21} & \textbf{95.29} & \textbf{94.0} & \textbf{95.31} & \textbf{93.95} & \textbf{1.93} & \textbf{91.62} & \textbf{94.57} & \textbf{93.59} & \textbf{94.73} & \textbf{93.63} & \textbf{1.43} & \textbf{97.97} & \textbf{98.8} & \textbf{98.47} & \textbf{98.77} & \textbf{98.5} & \textbf{0.38} \\

\hline
    \end{tabular}
}
\label{tab:race4}
\end{table*}

\begin{table*}[!ht]
\centering
\caption{Comparison of fairness \textit{w.r.t} the age of identity on BUPT-BalancedFace. The results are reported in $F_P$ metric}
\setlength{\tabcolsep}{10pt}
    \begin{tabular}{|l|ccccccccccc|}
    \hline
    \multirow{2}{*}{Method} & \multicolumn{11}{c|}{$F_P$}  \\ 
    \cline{2-12}
    & 0-2 & 3-9 & 10-19 & 20-29 & 30-39 & 40-49 & 50-59 & 60-69 & 70+ & Mean & STD \\

\hline 
\multicolumn{12}{|c|}{(C1) BUPT-In-the-wild (485K)} \\ 
\hline 
LTC~\cite{yang2019learning} & 62.5 & 68.41 & 75.46 & 77.19 & 77.32 & 75.79 & 75.82 & 74.05 & 68.57 & 72.79 & 5.12 \\
GCN-V~\cite{yang2020learning} & 67.12 & 73.3 & 79.59 & 81.45 & 81.62 & 79.82 & 79.82 & 78.17 & 72.75 & 77.07 & 4.93 \\
GCN-V/E~\cite{yang2020learning} & 70.75 & 77.23 & 82.93 & 84.83 & 84.9 & 83.0 & 82.83 & 81.31 & 75.88 & 80.41 & 4.79 \\
Clusformer~\cite{Nguyen_2021_CVPR} & 73.76 & 80.48 & 85.48 & 87.41 & 87.48 & 85.41 & 85.15 & 83.78 & 78.55 & 83.06 & 4.59 \\
STAR-FC~\cite{Shen_2021_CVPR} & 75.96 & 82.95 & 87.87 & 89.67 & 89.12 & 87.08 & 86.76 & 85.38 & 80.41 & 85.02 & 4.49 \\
OMHC~\cite{omhc} & 77.95 & 84.67 & 89.22 & 90.97 & 90.78 & 88.75 & 88.13 & 86.85 & 82.17 & 86.61 & 4.31 \\
Ours & \textbf{79.79} & \textbf{87.11} & \textbf{91.31} & \textbf{92.88} & \textbf{92.65} & \textbf{90.7} & \textbf{90.04} & \textbf{88.85} & \textbf{86.9} & \textbf{88.69} & \textbf{4.21} \\

\hline 
\multicolumn{12}{|c|}{(C2) BUPT-GlobalFace (2.2M)} \\ 
\hline 
Clusformer~\cite{Nguyen_2021_CVPR} & 76.09 & 82.94 & 87.66 & 89.52 & 89.35 & 87.3 & 86.8 & 85.52 & 80.52 & 85.08 & 4.46 \\
STAR-FC~\cite{Shen_2021_CVPR} & 79.2 & 86.07 & 90.46 & 92.02 & 91.83 & 89.83 & 89.24 & 88.08 & 83.88 & 87.85 & 4.18 \\
OMHC~\cite{omhc} & 79.75 & 87.33 & 91.12 & 92.7 & 92.58 & 90.57 & \textbf{90.02} & 88.99 & 84.68 & 88.64 & 4.19 \\
Ours & \textbf{81.99} & \textbf{88.41} & \textbf{92.97} & \textbf{94.16} & \textbf{93.29} & \textbf{90.6} & 89.86 & \textbf{90.15} & \textbf{85.48} & \textbf{89.66} & \textbf{3.93} \\

\hline 
\multicolumn{12}{|c|}{(C3) BUPT-In-the-wild (3.4M)} \\ 
\hline 
Clusformer~\cite{Nguyen_2021_CVPR} & 84.5 & 90.13 & 93.76 & 94.54 & 94.48 & 93.15 & 92.71 & 92.66 & 89.67 & 91.73 & 3.21 \\
STAR-FC~\cite{Shen_2021_CVPR} & 85.42 & 90.88 & 94.17 & 94.86 & 94.82 & 93.55 & 93.0 & 92.94 & 90.35 & 92.22 & 3.0 \\
OMHC~\cite{omhc} & 85.88 & 91.15 & 94.32 & 95.0 & 94.97 & 93.7 & 93.08 & 93.09 & 90.58 & 92.42 & 2.9 \\
Ours & \textbf{87.11} & \textbf{92.52} & \textbf{95.37} & \textbf{95.82} & \textbf{95.58} & \textbf{94.22} & \textbf{93.16} & \textbf{93.69} & \textbf{91.22} & \textbf{93.19} & \textbf{2.74} \\

\hline
\end{tabular}
\label{tab:age_fp}
\end{table*}

\begin{table*}[!ht]
\centering
\caption{Comparison of fairness \textit{w.r.t} the age of identity on BUPT-BalancedFace. The results are reported in $F_B$ metric}
\setlength{\tabcolsep}{10pt}
    \begin{tabular}{|l|ccccccccccc|}
    \hline
    \multirow{2}{*}{Method} & \multicolumn{11}{c|}{$F_B$}  \\ 
    \cline{2-12}
    & 0-2 & 3-9 & 10-19 & 20-29 & 30-39 & 40-49 & 50-59 & 60-69 & 70+ & Mean & STD \\

\hline 
\multicolumn{12}{|c|}{(C1) BUPT-In-the-wild (485K)} \\ 
\hline 
LTC~\cite{yang2019learning} & 69.47 & 74.87 & 79.12 & 78.4 & 77.82 & 76.43 & 77.89 & 78.46 & 75.85 & 77.36 & \textbf{2.97} \\
GCN-V~\cite{yang2020learning} & 72.77 & 78.46 & 82.61 & 82.19 & 81.7 & 80.01 & 81.22 & 81.52 & 78.62 & 80.79 & 3.06 \\
GCN-V/E~\cite{yang2020learning} & 75.4 & 81.4 & 85.44 & 85.22 & 84.7 & 82.87 & 83.75 & 83.88 & 80.75 & 83.5 & 3.14 \\
Clusformer~\cite{Nguyen_2021_CVPR} & 77.6 & 83.88 & 87.56 & 87.52 & 87.01 & 85.07 & 85.65 & 85.74 & 82.58 & 85.63 & 3.15 \\
STAR-FC~\cite{Shen_2021_CVPR} & 79.56 & 85.94 & 89.58 & 89.34 & 88.55 & 87.0 & 86.9 & 87.13 & 84.22 & 87.33 & 3.09 \\
OMHC~\cite{omhc} & 80.69 & 87.19 & 90.69 & 90.71 & 90.03 & 88.11 & 88.26 & 88.26 & 85.32 & 88.57 & 3.15 \\
Ours & \textbf{82.43} & \textbf{89.14} & \textbf{92.42} & \textbf{92.45} & \textbf{91.7} & \textbf{89.83} & \textbf{89.86} & \textbf{89.82} & \textbf{87.21} & \textbf{90.3} & 3.12 \\

\hline 
\multicolumn{12}{|c|}{(C2) BUPT-GlobalFace (2.2M)} \\ 
\hline 
Clusformer~\cite{Nguyen_2021_CVPR} & 79.31 & 85.82 & 89.38 & 89.39 & 88.71 & 86.79 & 87.12 & 87.16 & 84.05 & 87.3 & 3.17 \\
STAR-FC~\cite{Shen_2021_CVPR} & 81.66 & 88.28 & 91.7 & 91.67 & 90.98 & 89.09 & 89.18 & 89.22 & 86.48 & 89.58 & 3.13 \\
OMHC~\cite{omhc} & 82.41 & 89.19 & 92.32 & 92.37 & 91.68 & 89.62 & 89.78 & 89.95 & 87.23 & 90.27 & 3.1 \\
Ours & \textbf{84.75} & \textbf{90.72} & \textbf{93.82} & \textbf{93.78} & \textbf{92.97} & \textbf{91.14} & \textbf{90.97} & \textbf{91.3} & \textbf{88.75} & \textbf{90.91} & \textbf{2.82} \\

\hline 
\multicolumn{12}{|c|}{(C3) BUPT-In-the-wild (3.4M)} \\ 
\hline 
Clusformer~\cite{Nguyen_2021_CVPR} & 86.05 & 91.58 & 94.38 & 94.32 & 93.88 & 92.63 & 92.67 & 93.13 & 90.82 & 92.16 & 2.58 \\
STAR-FC~\cite{Shen_2021_CVPR} & 86.92 & 92.23 & 94.81 & 94.71 & 94.29 & 93.12 & 93.06 & 93.45 & 91.44 & 92.67 & 2.43 \\
OMHC~\cite{omhc} & 87.31 & 92.48 & 94.96 & 94.85 & 94.47 & 93.33 & 93.22 & 93.62 & 91.67 & 92.88 & 2.35 \\
Ours & \textbf{88.98} & \textbf{93.94} & \textbf{95.91} & \textbf{95.72} & \textbf{95.2} & \textbf{94.21} & \textbf{93.97} & \textbf{94.33} & \textbf{92.69} & \textbf{93.88} & \textbf{2.09} \\

\hline
    \end{tabular}
\label{tab:age_fb}
\end{table*}

\begin{table*}[!ht]
\centering
\caption{Comparison of fairness \textit{w.r.t} the age of identity on BUPT-BalancedFace. The results are reported in $NMI$ metric}
\setlength{\tabcolsep}{12pt}
    \begin{tabular}{|l|ccccccccccc|}
    \hline
    \multirow{2}{*}{Method} & \multicolumn{11}{c|}{$NMI$}  \\ 
    \cline{2-12}
    & 0-2 & 3-9 & 10-19 & 20-29 & 30-39 & 40-49 & 50-59 & 60-69 & 70+ & Mean & STD \\

\hline 
\multicolumn{12}{|c|}{(C1) BUPT-In-the-wild (485K)} \\ 
\hline 
LTC~\cite{yang2019learning} & 95.13 & 95.49 & 95.96 & 95.68 & 95.71 & 95.68 & 95.85 & 95.88 & 95.61 & 95.67 & \textbf{0.25} \\
GCN-V~\cite{yang2020learning} & 95.6 & 96.09 & 96.61 & 96.42 & 96.44 & 96.3 & 96.44 & 96.43 & 96.08 & 96.27 & 0.3 \\
GCN-V/E~\cite{yang2020learning} & 95.97 & 96.58 & 97.14 & 97.02 & 97.0 & 96.8 & 96.89 & 96.86 & 96.44 & 96.74 & 0.36 \\
Clusformer~\cite{Nguyen_2021_CVPR} & 96.29 & 97.0 & 97.55 & 97.48 & 97.44 & 97.19 & 97.23 & 97.19 & 96.75 & 97.12 & 0.4 \\
STAR-FC~\cite{Shen_2021_CVPR} & 96.41 & 97.25 & 97.91 & 97.63 & 97.94 & 97.31 & 97.52 & 97.25 & 96.98 & 97.36 & 0.48 \\
OMHC~\cite{omhc} & 96.72 & 97.58 & 98.15 & 98.12 & 98.01 & 97.74 & 97.7 & 97.67 & 97.22 & 97.66 & 0.46 \\
Ours & \textbf{96.9} & \textbf{97.92} & \textbf{98.49} & \textbf{98.46} & \textbf{98.32} & \textbf{98.03} & \textbf{97.98} & \textbf{97.95} & \textbf{97.54} & \textbf{97.95} & 0.5 \\

\hline 
\multicolumn{12}{|c|}{(C2) BUPT-GlobalFace (2.2M)} \\ 
\hline 
Clusformer~\cite{Nguyen_2021_CVPR} & 96.53 & 97.34 & 97.9 & 97.85 & 97.76 & 97.5 & 97.5 & 97.46 & 97.0 & 97.43 & 0.44 \\
STAR-FC~\cite{Shen_2021_CVPR} & 96.84 & 97.78 & 98.35 & 98.31 & 98.19 & 97.91 & 97.87 & 97.84 & 97.42 & 97.83 & 0.47 \\
OMHC~\cite{omhc} & 96.7 & 97.83 & 98.25 & \textbf{98.66} & 98.17 & 98.0 & 97.8 & 97.79 & 97.59 & 97.86 & 0.54 \\
Ours & \textbf{97.72} & \textbf{98.11} & \textbf{98.73} & 98.66 & \textbf{98.47} & \textbf{98.13} & \textbf{98.06} & \textbf{98.15} & \textbf{97.65} & \textbf{98.19} & \textbf{0.38} \\

\hline 
\multicolumn{12}{|c|}{(C3) BUPT-In-the-wild (3.4M)} \\ 
\hline 
Clusformer~\cite{Nguyen_2021_CVPR} & 97.57 & 98.33 & 98.78 & 98.77 & 98.71 & 98.52 & 98.51 & 98.58 & 98.21 & 98.44 & 0.38 \\
STAR-FC~\cite{Shen_2021_CVPR} & 97.71 & 98.45 & 98.87 & 98.85 & 98.8 & 98.62 & 98.58 & 98.64 & 98.32 & 98.54 & 0.36 \\
OMHC~\cite{omhc} & 97.77 & 98.5 & 98.9 & 98.88 & 98.84 & 98.66 & 98.61 & 98.67 & 98.36 & 98.58 & 0.35 \\
Ours & \textbf{98.03} & \textbf{98.8} & \textbf{99.12} & \textbf{99.08} & \textbf{98.99} & \textbf{98.83} & \textbf{98.74} & \textbf{98.81} & \textbf{98.53} & \textbf{98.77} & \textbf{0.33} \\

\hline
    \end{tabular}
\label{tab:age_nmi}
\end{table*}

\begin{table*}[!ht]
\centering
\caption{Comparison of fairness \textit{w.r.t} the pseudo race (7 classes) of identity on BUPT-BalancedFace. The results are reported in $F_P$ metric. EA: \textit{East Asian}, SA: \textit{Southeast Asian}, LH: \textit{Latino Hispanic}, MD: \textit{Middle Eastern}}
\setlength{\tabcolsep}{12pt}
    \begin{tabular}{|l|ccccccccc|}
    \hline
    \multirow{2}{*}{Method} & \multicolumn{9}{c|}{$F_P$} \\
    \cline{2-10}
    & EA & SA & LH & B & I & ME & W & Mean & STD \\

\hline 
\multicolumn{10}{|c|}{(C1) BUPT-In-the-wild (485K)} \\ 
\hline 
LTC~\cite{yang2019learning} & 61.66 & 63.5 & 74.39 & 78.32 & 82.12 & 74.9 & 78.86 & 73.39 & 7.84 \\
GCN-V~\cite{yang2020learning} & 67.2 & 68.75 & 78.75 & 82.66 & 85.2 & 79.41 & 82.88 & 77.84 & 7.09 \\
GCN-V/E~\cite{yang2020learning} & 71.8 & 72.97 & 82.15 & 85.78 & 87.87 & 82.89 & 86.16 & 81.37 & 6.45 \\
Clusformer~\cite{Nguyen_2021_CVPR} & 75.67 & 76.41 & 84.81 & 88.15 & 89.78 & 85.57 & 88.44 & 84.12 & 5.78 \\
STAR-FC~\cite{Shen_2021_CVPR} & 78.8 & 78.96 & 86.93 & 90.01 & 91.36 & 87.72 & 90.28 & 86.29 & 5.29 \\
OMHC~\cite{omhc} & 81.11 & 81.29 & 88.36 & 91.26 & 92.47 & 89.15 & 91.43 & 87.87 & 4.76 \\
Ours & \textbf{86.62} & \textbf{86.49} & \textbf{90.24} & \textbf{92.9} & \textbf{93.94} & \textbf{91.12} & \textbf{93.18} & \textbf{90.07} & \textbf{3.97} \\

\hline 
\multicolumn{10}{|c|}{(C2) BUPT-GlobalFace (2.2M)} \\ 
\hline 
Clusformer~\cite{Nguyen_2021_CVPR} & 78.74 & 79.1 & 86.85 & 89.93 & 91.41 & 87.61 & 90.12 & 86.25 & 5.24 \\
STAR-FC~\cite{Shen_2021_CVPR} & 83.01 & 83.05 & 89.46 & 92.19 & 93.33 & 90.28 & 92.4 & 89.1 & 4.35 \\
OMHC~\cite{omhc} & 84.67 & 84.55 & \textbf{90.41} & 92.65 & 93.79 & \textbf{91.3} & 93.31 & 90.1 & 3.92 \\
Ours & \textbf{88.9} & \textbf{88.49} & 89.57 & \textbf{93.4} & \textbf{94.68} & 91.24 & \textbf{94.62} & \textbf{91.0} & \textbf{3.36} \\

\hline 
\multicolumn{10}{|c|}{(C3) BUPT-In-the-wild (3.4M)} \\ 
\hline 
Clusformer~\cite{Nguyen_2021_CVPR} & 90.13 & 89.08 & 92.54 & 94.84 & 95.63 & 93.49 & 95.54 & 93.04 & 2.61 \\
STAR-FC~\cite{Shen_2021_CVPR} & 90.77 & 89.87 & 92.91 & 95.13 & 95.86 & 93.91 & 95.86 & 93.47 & 2.41 \\
OMHC~\cite{omhc} & 91.08 & 90.14 & 93.08 & 95.21 & 95.95 & 94.06 & 95.97 & 93.64 & 2.33 \\
Ours & \textbf{93.81} & \textbf{93.0} & \textbf{93.53} & \textbf{95.49} & \textbf{96.42} & \textbf{94.61} & \textbf{96.39} & \textbf{94.44} & \textbf{1.79} \\

\hline
\end{tabular}
\label{tab:race7_fp}
\end{table*}

\begin{table*}[!ht]
\centering
\caption{Comparison of fairness \textit{w.r.t} the pseudo race (7 classes) of identity on BUPT-BalancedFace. The results are reported in $F_B$ metric. EA: \textit{East Asian}, SA: \textit{Southeast Asian}, LH: \textit{Latino Hispanic}, MD: \textit{Middle Eastern}}
\setlength{\tabcolsep}{12pt}
    \begin{tabular}{|l|ccccccccc|}
    \hline
    \multirow{2}{*}{Method} & \multicolumn{9}{c|}{$F_B$} \\
    \cline{2-10}
    & EA & SA & LH & B & I & ME & W & Mean & STD \\

\hline 
\multicolumn{10}{|c|}{(C1) BUPT-In-the-wild (485K)} \\ 
\hline 
LTC~\cite{yang2019learning} & 67.86 & 67.93 & 73.57 & 78.79 & 82.89 & 75.78 & 79.19 & 75.14 & 5.74 \\
GCN-V~\cite{yang2020learning} & 72.42 & 72.38 & 77.54 & 82.58 & 85.66 & 79.64 & 82.85 & 79.01 & 5.19 \\
GCN-V/E~\cite{yang2020learning} & 76.22 & 75.96 & 80.66 & 85.38 & 88.02 & 82.62 & 85.8 & 82.09 & 4.73 \\
Clusformer~\cite{Nguyen_2021_CVPR} & 79.32 & 78.83 & 83.08 & 87.57 & 89.71 & 84.96 & 87.87 & 84.48 & 4.26 \\
STAR-FC~\cite{Shen_2021_CVPR} & 81.89 & 80.92 & 84.93 & 89.35 & 91.28 & 86.92 & 89.37 & 86.38 & 3.96 \\
OMHC~\cite{omhc} & 83.68 & 82.99 & 86.36 & 90.35 & 92.09 & 88.17 & 90.73 & 87.77 & 3.55 \\
Ours & \textbf{86.54} & \textbf{85.61} & \textbf{88.08} & \textbf{91.82} & \textbf{93.4} & \textbf{89.93} & \textbf{92.36} & \textbf{89.68} & \textbf{3.02} \\

\hline 
\multicolumn{10}{|c|}{(C2) BUPT-GlobalFace (2.2M)} \\ 
\hline 
Clusformer~\cite{Nguyen_2021_CVPR} & 81.74 & 81.14 & 84.96 & 89.16 & 91.12 & 86.8 & 89.48 & 86.34 & 3.89 \\
STAR-FC~\cite{Shen_2021_CVPR} & 85.29 & 84.43 & 87.34 & 91.2 & 92.86 & 89.17 & 91.66 & 88.85 & 3.27 \\
OMHC~\cite{omhc} & 86.41 & 85.64 & 88.05 & 91.94 & 93.15 & 89.68 & 92.13 & 89.57 & 2.96 \\
Ours & \textbf{89.1} & \textbf{87.8} & \textbf{89.41} & \textbf{92.84} & \textbf{94.26} & \textbf{91.36} & \textbf{93.85} & \textbf{91.23} & \textbf{2.53} \\

\hline 
\multicolumn{10}{|c|}{(C3) BUPT-In-the-wild (3.4M)} \\ 
\hline 
Clusformer~\cite{Nguyen_2021_CVPR} & 90.78 & 89.8 & 91.08 & 94.19 & 95.28 & 92.66 & 94.94 & 92.68 & 2.18 \\
STAR-FC~\cite{Shen_2021_CVPR} & 91.43 & 90.54 & 91.61 & 94.55 & 95.56 & 93.22 & 95.33 & 93.18 & 2.03 \\
OMHC~\cite{omhc} & 91.74 & 90.82 & 91.84 & 94.65 & 95.67 & 93.44 & 95.48 & 93.38 & 1.95 \\
Ours & \textbf{93.39} & \textbf{92.38} & \textbf{92.78} & \textbf{95.16} & \textbf{96.22} & \textbf{94.4} & \textbf{96.11} & \textbf{94.35} & \textbf{1.56} \\

\hline
    \end{tabular}
\label{tab:race7_fb}
\end{table*}

\begin{table*}[!ht]
\centering
\caption{Comparison of fairness \textit{w.r.t} the pseudo race (7 classes) of identity on BUPT-BalancedFace. The results are reported in $NMI$ metric. EA: \textit{East Asian}, SA: \textit{Southeast Asian}, LH: \textit{Latino Hispanic}, MD: \textit{Middle Eastern}}
\setlength{\tabcolsep}{12pt}
    \begin{tabular}{|l|ccccccccc|}
    \hline
    \multirow{2}{*}{Method} & \multicolumn{9}{c|}{$NMI$} \\
    \cline{2-10}
    & EA & SA & LH & B & I & ME & W & Mean & STD \\

\hline 
\multicolumn{10}{|c|}{(C1) BUPT-In-the-wild (485K)} \\ 
\hline 
LTC~\cite{yang2019learning} & 93.56 & 93.91 & 95.03 & 95.41 & 96.45 & 95.48 & 95.5 & 95.05 & 1.0 \\
GCN-V~\cite{yang2020learning} & 94.4 & 94.67 & 95.73 & 96.2 & 97.03 & 96.17 & 96.27 & 95.78 & 0.94 \\
GCN-V/E~\cite{yang2020learning} & 95.11 & 95.29 & 96.3 & 96.78 & 97.51 & 96.71 & 96.88 & 96.37 & 0.88 \\
Clusformer~\cite{Nguyen_2021_CVPR} & 95.71 & 95.81 & 96.73 & 97.23 & 97.86 & 97.13 & 97.34 & 96.83 & 0.8 \\
STAR-FC~\cite{Shen_2021_CVPR} & 96.02 & 96.22 & 96.97 & 97.38 & 98.1 & 97.29 & 97.43 & 97.06 & 0.73 \\
OMHC~\cite{omhc} & 96.55 & 96.56 & 97.33 & 97.81 & 98.34 & 97.71 & 97.95 & 97.46 & 0.69 \\
Ours & \textbf{97.11} & \textbf{97.02} & \textbf{97.61} & \textbf{98.12} & \textbf{98.6} & \textbf{98.03} & \textbf{98.3} & \textbf{97.83} & \textbf{0.6} \\

\hline 
\multicolumn{10}{|c|}{(C2) BUPT-GlobalFace (2.2M)} \\ 
\hline 
Clusformer~\cite{Nguyen_2021_CVPR} & 96.17 & 96.23 & 97.07 & 97.56 & 98.14 & 97.47 & 97.68 & 97.19 & 0.75 \\
STAR-FC~\cite{Shen_2021_CVPR} & 96.87 & 96.82 & 97.5 & 97.99 & 98.49 & 97.9 & 98.14 & 97.67 & 0.64 \\
OMHC~\cite{omhc} & 97.03 & 97.16 & 97.45 & 97.89 & 98.62 & \textbf{98.16} & 98.53 & 97.84 & 0.64 \\
Ours & \textbf{97.56} & \textbf{97.31} & \textbf{97.63} & \textbf{98.23} & \textbf{98.73} & 98.16 & \textbf{98.54} & \textbf{98.02} & \textbf{0.53} \\

\hline 
\multicolumn{10}{|c|}{(C3) BUPT-In-the-wild (3.4M)} \\ 
\hline 
Clusformer~\cite{Nguyen_2021_CVPR} & 97.78 & 97.72 & 98.17 & 98.63 & 98.96 & 98.55 & 98.82 & 98.38 & 0.49 \\
STAR-FC~\cite{Shen_2021_CVPR} & 97.93 & 97.87 & 98.28 & 98.71 & 99.02 & 98.65 & 98.91 & 98.48 & 0.46 \\
OMHC~\cite{omhc} & 98.0 & 97.94 & 98.32 & 98.73 & 99.05 & 98.69 & 98.95 & 98.53 & 0.44 \\
Ours & \textbf{98.44} & \textbf{98.31} & \textbf{98.5} & \textbf{98.84} & \textbf{99.17} & \textbf{98.87} & \textbf{99.09} & \textbf{98.75} & \textbf{0.33} \\

\hline
    \end{tabular}
\label{tab:race7_nmi}
\end{table*}

\begin{table}[!ht]
\small
\centering
\caption{Effects of Different settings on
numbers of sub-clusters.}
\label{tab:ablationstudy_sk}
\resizebox{1.\columnwidth}{!}{
    \begin{tabular}{@{}|l|cccc|ll|@{}}
    
    \hline
    $k$ & Asian & African & Caucasian & Indian & $F_P$ & $std$ \\
    \hline
    4 & 89.65 & 94.04 & 95.15 & 94.28 & \textbf{93.28} & \textbf{2.47} \\
    5 & 89.13 & 94.26 & 95.14 & 94.07 & 93.15 & 2.72 \\
    8 & 88.83 & 94.21 & 95.06 & 93.94 & 93.01 & 2.83 \\
    10 & 88.52 & 94.12 & 94.95 & 93.79 & 92.85 & 2.92 \\
    16 & 87.94 & 93.96 & 94.72 & 93.44 & 92.52 & 3.09 \\
    \hline
    \end{tabular}
}
\end{table}

\subsubsection{Fairness on different demographics}

We study the fairness of our methods in different demographics such as age, gender, and race. In particular, the pre-trained model on FairFace \cite{karkkainenfairface} is employed to predict these demographics for each subject in the BUPT-BalancedFace database. A threshold of 0.9 is applied to filter out the uncertain samples and keep only the most confident images for evaluation. We conduct experiments and measure fairness on nine ranges of age (0-2, 3-9, 10-19, 20-29, 30-39, 40-49, 50-59, 60-69, and 70+), two groups of gender (Male and Female) and seven racial groups (East Asian, Southeast Asian, Latino Hispanic, Black, Indian, Middle Eastern, White). These results are reported in Fig \ref{fig:age}, Fig \ref{fig:gender}, and Fig \ref{fig:race_7}, respectively.

In addition, we also show the detailed experiment results of clustering performance on different demographics in the table style. Table \ref{tab:enthic}, table \ref{tab:gender}, table \ref{tab:race4}, table \ref{tab:age_fp} \ref{tab:age_fb} \ref{tab:age_nmi}, table \ref{tab:race7_fp} \ref{tab:race7_fb} \ref{tab:race7_nmi} show the full results ($F_P$, $F_B$ and $NMI$) of methods' on BUPT-BalancedFace with respect to different demographics. It is noted that table \ref{tab:enthic} and table \ref{tab:race4} share the same meaning of demographic, i.e., ethnicity, but the labels for ethnicity in \ref{tab:race4} are inferred by a trained FairFace model. Both two tables demonstrate our superior performance compared to previous works.

\subsubsection{Ablation studies}

In this section, we analyze the roles of C-ATT, $\mathcal{L}^{FMI}$ and $\mathcal{L}^{fair}$ on their contributions (see Table \ref{tab:ablationstudy}) and experiment Intraformer with different decomposition settings (see Table \ref{tab:ablationstudy_sk}).
We train the model on BUPT-In-the-wild (3.4M) and evaluate 
on the ethnicity aspect of BUPT-BalancedFace.

\noindent
\textbf{The roles of C-ATT, $\mathcal{L}^{FMI}$ and $\mathcal{L}^{fair}$}. We configure a baseline Intraformer with four sub-clusters using neither C-ATT nor $\mathcal{L}^{FMI}$ nor $\mathcal{L}^{fair}$. 
Unsurprisingly, our baseline achieves an $F_P$ of 91.21\% and $std$ of 3.89\%, which is much lower than previous methods, i.e., Clusformer, STAR-FC and OMHC. 
With $\mathcal{L}^{FMI}$ loss, an improvement of 0.65\% for $F_P$ and 0.3\% for $std$ is achieved. In the third experiment, $\mathcal{L}^{fair}$ is employed to train the model concurrently, and both $F_P$ and $std$ values are improved by 0.8\% and 0.5\%, respectively.  
When C-ATT can further explore correlations between centroid to all samples and encourage hard clusters toward the fairness point, the performance surpasses the state-of-the-art method.

\noindent
\textbf{Stabilization of Intraformer}. There are several ways to decompose a cluster in the dataset into smaller sub-clusters. Table \ref{tab:ablationstudy} shows the performance of Intraformer in different settings. It is important to note that the performance of Intraformer does not fluctuate much across different settings. The average $F_P$ score is 93.28\% $\pm$ 2.47\%. These results emphasize the stabilization and robustness of our approach.

\begin{table}[!t]
\small
\centering
\caption{Ablation Study on  C-ATT, $\mathcal{L}^{fair}$ and $\mathcal{L}^{FMI}$}
\label{tab:ablationstudy}
\resizebox{1.0\columnwidth}{!}{
\begin{tabular}{@{}|lll|cccc|ll|@{}}
\hline
C-ATT & $\mathcal{L}^{fair}$ & $\mathcal{L}^{FMI}$ & Asian & African & Caucasian & Indian & $F_P$ & $std$ \\
\hline

\xmark & \xmark & \xmark & 85.46 & 93.29 & 93.83 & 92.27 & 91.21 & 3.89 \\
\xmark & \xmark & \cmark & 86.68 & 93.64 & 94.29 & 92.81 & 91.86 & 3.50 \\
\xmark & \cmark & \cmark & 88.23 & 94.05 & 94.84 & 93.63 & 92.69 & 3.01 \\
\cmark & \cmark & \cmark & 89.65 & 94.04 & 95.15 & 94.28 & \textbf{93.28} &\textbf{ 2.47} \\
\hline
\end{tabular}
}
\end{table}

\section{Conclusions and Discussions}

Unlike a few concurrent research on the fairness of computer vision, we study how to address the unfair facial clustering problem. We introduce cluster purity as an indicator of demographic bias. Secondly, we propose a novel loss to enforce the consistency of purity between clusters of different groups; thus, fairness can be achieved. Finally, a novel framework for visual clustering was presented to strengthen the hard clusters, which usually come from the minor/biased group. This framework contributes not only to fairness but also to clustering performance overall.\\
\textbf{Limitations}. We formulate the fairness problem for the face clustering as in Eqn \eqref{eq:upper_bound_objective_function}. Fairness can be achieved by optimizing the upper bound as in Eqn \eqref{eq:upper_bound_objective_function}. However, there is a gap in performance between optimizing by Eqn \eqref{eq:upper_bound_objective_function} and by Eqn \eqref{eq:fairness_object_tive}. This problem is further investigated in future work.

\newpage
\bibliographystyle{IEEEtran}
\bibliography{sn-bibliography}%
\end{document}